\documentclass[11pt, a4paper, onecolumn, copyright]{AweAI}

\usepackage[authoryear, sort&compress, round]{natbib}
\usepackage[utf8]{inputenc}

\usepackage[utf8]{inputenc}
\usepackage[T1]{fontenc}
\usepackage{hyperref}
\usepackage{url}
\usepackage{booktabs}
\usepackage{amsfonts}
\usepackage{nicefrac}
\usepackage{microtype}
\usepackage{xcolor}
\usepackage{amsmath,amssymb,amsthm}
\usepackage{multirow}
\usepackage{graphicx}
\usepackage{subcaption}
\usepackage{makecell}
\usepackage{listings}
\usepackage{algorithm}
\usepackage{algpseudocode}
\usepackage{tikz}
\usepackage{marvosym}
\usetikzlibrary{arrows.meta,positioning,fit,backgrounds}
\usepackage{pgfplots}
\usepackage{enumitem}
\usepgfplotslibrary{groupplots}
\pgfplotsset{compat=1.18}

\newtheorem{theorem}{Theorem}
\newtheorem{proposition}[theorem]{Proposition}

\theoremstyle{definition}

\newtheorem{assumption}{Assumption}
\theoremstyle{remark}

\newcommand{\methodname}{TacoMAS}           

\newcommand{\G}{\mathcal{G}}
\newcommand{\V}{\mathcal{V}}
\newcommand{\E}{\mathcal{E}}
\newcommand{\acc}{\text{Acc}}
\newcommand{\calls}{\text{Calls}}

\definecolor{promptkey}{rgb}{0.10,0.30,0.60}
\definecolor{promptstr}{rgb}{0.55,0.30,0.10}
\lstdefinelanguage{prompt}{
  morekeywords={System,User,Assistant,Role,Goal,Constraint,Action,Observation,Reflection,Plan},
  morecomment=[l]{\#},
  morestring=[b]",
  sensitive=true,
}
\title{\includegraphics[height=1.2em]{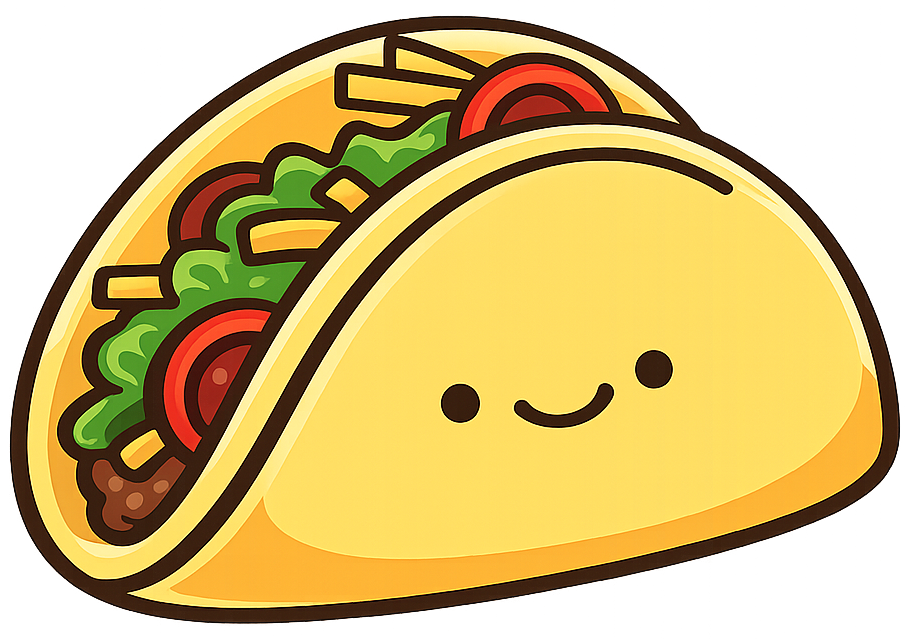}\hspace{0.01em} TacoMAS: Test-Time Co-Evolution of Topology 
and Capability in LLM-based Multi-Agent Systems}

\author{%
  {Chen Xu$^1$~Yicheng Hu$^2$~Ruizi Wang$^2$~~Xinyu Lin$^3$\textsuperscript{\Letter}~~Wenjie Wang$^2$\textsuperscript{\Letter}~Dongrui Liu$^4$~Fuli Feng$^2$}\\
  $^1$Carnegie Mellon University  $^2$University of Science and Technology of China \\$^3$National University of Singapore $^4$Shanghai AI Lab\\
  \texttt{chenxu0427ruc@gmail.com, xylin1028@gmail.com, wenjiewang96@gmail.com}
}

\renewcommand{\today}{}
\newcommand{\githubrepo}{https://github.com/chenxu2-gif/TacoMAS-MultiAgent}
\makeatletter
\renewcommand{\abscontent}{%
    \noindent
    \parbox{\dimexpr\linewidth}{\absfont \theabstract}%
    \vskip1em
    \noindent GitHub: \href{\githubrepo}{\texttt{\githubrepo}}\\
        \vspace{0.4em}
    \hrule
    \vspace{0.3em}
    {\small
    \textsuperscript{\Letter} Corresponding authors.
    }
}
\fancypagestyle{firststyle}{

    \fancyhead[R]{%
        \ifdefined\paperurl
        \if\relax\the\paperurl\relax \else
            \href{\the\paperurl}{\urlheaderfont \itshape \the\paperurl}\\ \fi
        \else
        \fi
        {\footerfont\itshape\monthyeardate\today\hspace{8pt}}%
        \ifthenelse{\boolean{confidential}}{\\ \footerfont \internalonly}{}
    }

    \fancyhead[C]{}

    \fancyfoot[L]{
    \footerfont\itshape
        \ifdefined\correspondingauthor
        \if\relax\detokenize\expandafter{\the\correspondingauthor}\relax
      \else
        \footerfont\itshape {Emails: \the\correspondingauthor\\}%
      \fi
    \fi
    }

    \fancyfoot[R]{
        \ifthenelse{\boolean{confidential}}{
        \ifdefined\reportnumber
        \if\relax\the\reportnumber\relax
        \else \footerfont\itshape  {\footerfont \thepa{} Technical Report \the\reportnumber} \fi
        \else \fi
        }{\footerfont\bfseries\relax}
    }
    \fancyfoot[C]{
        \footerfont\bfseries\relax}
}
\makeatother

\begin{abstract}
Multi-agent systems (MAS) have emerged as a promising paradigm for solving complex tasks. Recent work has explored self-evolving MAS that automatically optimize agent capabilities or communication topologies. However, existing methods either learn a topology that remains fixed at inference time or adapt only topology or capability during inference. We empirically and theoretically show that effective test-time evolution requires jointly adapting both axes, but on different time scales: capabilities should update rapidly to handle emerging subtasks, while the topology should evolve more slowly to preserve coordination stability.
We then introduce \methodname{}, a test-time co-evolution framework for dynamic MAS. \methodname{} formulates MAS inference as a task of online graph adaptation, where nodes represent agents with role-specific capability and edges define their communication topology. 
During inference, a fast capability loop updates agent expertise using trajectory-level feedback, 
while a slow meta-LLM-driven topology loop performs agents' birth-death operations on MAS, including edge edit, agent addition, and agent removal. 
We further show that this fast–slow design drives MAS evolution toward a task-conditioned stable equilibrium. 
Experiments on four benchmarks demonstrate that \methodname{} outperforms nearly 20 multi-agent baselines, achieving an average improvement of 13.3\% over the strongest baseline.
\end{abstract}

\begin{document}

\begingroup
\makeatletter
\renewcommand{\thefootnote}{}
\renewcommand{\@makefnmark}{}
\maketitle
\makeatother
\endgroup




\section{Introduction}
\label{sec:intro}

Recent advances in large language models (LLMs) have enabled increasingly capable autonomous agents, yet many real-world problems remain too complex for a single agent to solve reliably~\citep{wang2024survey,guo2024large,handler2023balancing}. Tasks such as software engineering, retrieval-intensive analysis, and long-horizon planning often require decomposing a problem into multiple interdependent subtasks. Multi-agent systems (MAS) provide a natural solution by coordinating specialized agents with different roles and capabilities. However, their effectiveness depends critically on how agents are organized and how responsibilities are allocated. Therefore, a growing line of research argues that the topology and capabilities of MAS should not be manually fixed, but automatically optimized or evolved for different tasks~\citep{li2024survey,piccialli2025agentai}.

Previous work on evolving MAS can be broadly divided into training-time and test-time approaches. Training-time methods optimize the agent topology or role assignment once and keep it fixed during inference~\citep{hong2023metagpt, zhang2024aflow, shang2024agentsquare, wang2025evoagentx, hu2024automated}. However, because the learned topology is fixed, it can easily mismatch unseen tasks whose latent subtasks and coordination demands deviate from the training distribution. 
Test-time methods instead treat inference as a dynamic evolution process, allowing MAS to adjust online based on intermediate states~\citep{qian2024chatdev,tastan2026stochastic, qu2026coral}. However, existing methods typically evolve either the communication topology~\citep{qian2024chatdev,tastan2026stochastic} or agent capabilities~\citep{qu2026coral} alone. In fact, optimizing both aspects is essential; it is a key prerequisite for unlocking the full collaborative potential of MAS~\citep{kim2025towards}.
This raises a question: how can we jointly adapt both topology and capabilities of MAS during inference?

However, naively combining these directions by updating both topology and capability online is problematic~\citep{papoudakis2021agent}. Evolving the two at the same pace can cause local adaptation to destabilize global coordination (see the theoretical and empirical evidence in \S~\ref{sec:theory-joint} and \ref{sec:exp_analysis}, respectively). For example, when an intermediate error is detected, a verifier agent may need to rapidly strengthen its checking capability. But if the topology is simultaneously rewired, the evidence flow and role dependencies underpinning the verifier agent may shift, turning a useful local update into a system-level failure. This motivates a natural fast–slow separation~\citep{fabiano2021epistemic,mguni2023mansa}, in which capability evolves on the fast timescale and topology on the slow one.

This fast-slow separation is not merely an engineering choice, but follows from evolutionary game theory. We model capability evolution as replicator dynamics over agent strategies and topology evolution as a slower adaptive process over the interaction graph. Together, they form a two-timescale replicator (\textit{i.e., }mutator system), where the fast process tracks the Evolutionarily Stable Strategy (ESS)~\citep{smith1973logic, tayloreshel1978} under the current topology and the slow process updates against this equilibrium response~\citep{borkar1997stochastic, kushneryin2003, nowaksigmund2004}. Intuitively, ESS means that the team has reached a locally stable division of labor, where each agent's capability and interaction pattern are well matched to the task and resistant to small deviations.

Motivated by this principle, we propose \methodname{}, which adapts both \textbf{T}opology and c\textbf{a}pability in a \textbf{co}-evolution framework for \textbf{MAS} during the inference of each query (Fig.~\ref{fig:method}). It consists of 
(1) a fast capability loop, where agents optimize their expertise based on their execution outcomes and contribution to the task in each round\footnote{In practice, this capability refinement can be implemented via updating contextual memory, refining role-specific instructions, or fine-tuning model parameters. Here we just use the contextual memory as an example.}; and
(2) a slow meta-LLM-driven topology loop, which periodically reviews the trajectory and proposes a birth-death (BD) update with a small set of edge and agent edits. 
During the BD process, the meta-LLM decides which edges in the agent topology should be modified and whether to introduce a new agent or remove an ineffective one. In this way, the inference process is guided toward an ESS, as theoretically justified in \S~\ref{sec:theory}.



Following the standard setup of recent multi-agent studies~\citep{kim2025towards}, we evaluate \methodname{} on four benchmarks spanning diverse task regimes: financial problem analysis, web browsing, Minecraft-style planning, and workplace task execution. Compared with nearly 20 MAS baselines, \methodname{} achieves an average improvement of 13.3\% over the strongest baseline across the four datasets.

In summary, our key contributions are three-fold:
\begin{enumerate}[leftmargin=*]
    \item We highlight a key principle for test-time multi-agent evolution: agent capabilities and team topology should be adapted jointly, but on different time scales.

    \item We propose \methodname{}, a test-time co-evolution framework that jointly adapts node capabilities and graph topology through two coupled loops. We further provide a theoretical analysis connecting this fast-slow design, showing convergence under bounded edit rates (\S\ref{sec:theory}).

    \item We conduct extensive experiments on four benchmarks. \methodname{} achieves the best performance on all datasets with an average improvement of $13.3\%$ over the strongest baselines.
\end{enumerate}

\section{Related Work}
\label{sec:related}

\paragraph{Multi-agent LLM systems.}
The shift from single LLM
agents~\citep{yao2022react,shinn2023reflexion,schick2023toolformer}
to multi-agent systems was motivated by tasks that demand
specialized roles and inter-agent coordination, e.g., long-horizon
software development, retrieval-heavy financial analysis, and
multi-step planning~\citep{zhou2024webarena,jimenez2024swebench,wei2025browsecomp}.
The first generation of multi-agent frameworks coordinates a
hand-crafted team of role-specialized agents:
AutoGen~\citep{wu2024autogen} and MetaGPT~\citep{hong2023metagpt}
ship role templates and standardized operating procedures;
CAMEL~\citep{li2023camel} pairs a user agent with an assistant in
a fixed dialogue loop; AgentVerse~\citep{chen2023agentverse} and
ChatDev~\citep{qian2024chatdev} assembles role rosters per task
category. \emph{Limitation:} the graph and roster are designed
once and held fixed; mid-instance signals cannot trigger new
roles or rewiring.

\paragraph{Training- / Offline-evolving multi-agent systems.}
A second line replaces the human designer with automated search
or learning, but the resulting artifact is still frozen at
inference. Two families dominate.
(i) Offline workflow/agent search produces one
graph that all test queries share:
AFlow~\citep{zhang2024aflow} explores workflow graphs with MCTS,
AgentSquare~\citep{shang2024agentsquare} searches a modular
``planning/reasoning/memory/tool-use'' design space,
ADAS~\citep{hu2024automated} alternates a code-space designer with an
executor, and EvoAgentX~\citep{dang2025multiagentcollaboration} mutates agent
populations with evolutionary search.
(ii) Trained per-query graph generators train a
conditional generator once, then sample (and freeze) a fresh
graph for each query: ARG-Designer~\citep{li2026assemble}
autoregressively emits a DAG; MaAS~\citep{zhang2025multi} samples
from a learned agentic supernet; MetaAgent~\citep{zhang2025metaagent}
predicts an FSM of agent transitions;
SwarmAgentic~\citep{zhang2025swarmagentic} assembles teams via a
particle-swarm metaphor; MetaGen~\citep{wang2026metagen} and
EvolveRouter~\citep{huang2026evolverouter} likewise regenerate the
roster / routing per query with only constrained execution-time
edits.
\emph{Limitation:} both families pay the design cost \emph{once}
and then freeze the artifact at inference; whichever graph looked
best at design/sampling time cannot react to evidence that
surfaces only after a few rounds of solving the actual instance.

\paragraph{Test-time evolving multi-agent systems.}
A growing line updates the MAS
during an instance, treating ``inference time'' as a
dynamic process. Existing methods, however, each commit to a
single update axis.
(i) Topology-only: ChatDev-Puppeteer~\citep{dang2025multiagentcollaboration}
has a centralized orchestrator pick the next persona over a fixed
pool; SelfOrg~\citep{tastan2026stochastic} rebuilds a top-$k$
communication DAG every round from response-similarity Shapley
scores. In both, agent prompts and tool policies are fixed.
(ii) Capability-only: CORAL~\citep{qu2026coral} updates a shared
memory and skill bank in a long-running loop, while the
topology stays implicit. 
Crucially, either research line fails to exploit the complete potential of multi-agent collaboration. \methodname\ fills this gap as the first to explore the joint optimization of topology and capability within a single inference. We empirically and theoretically demonstrate that their co-evolutionary interaction is essential for maximizing performance. To formalize this, we leverage evolutionary game theory~\citep{tayloreshel1978,nowaksigmund2004,hofbauer1998evolutionary,akin1979geometry} and two-time-scale stochastic approximation~\citep{borkar1997stochastic,kushneryin2003} as our analytical machinery in \S\ref{sec:theory}.


\section{Method: \methodname{}}
\label{sec:method}

The overview of our proposed framework is illustrated in Figure~\ref{fig:method} and the complete procedure of \methodname\ is summarized in Algorithm~\ref{alg:tc-evo}.


\subsection{Setup and Notation}
\label{sec:method-formal}

\paragraph{Multi-agent system setting.}
Given a query $q$, MAS generate an answer $a$ through a complete forward workflow, \textit{i.e.,} a round of MAS execution. This workflow is defined by the system's configuration, including its agent roles, individual capabilities, and communication topology. Different MAS frameworks adopt varied designs for these components to optimize task performance.

\paragraph{Test-time evolution.}
Unlike static systems, we perform an online evolution of the MAS during the inference of each query. We formalize the system as a directed agent graph $\mathcal{G}_t = (T_t, \Phi_t)$ indexed by execution round $t=0,1,\dots,R$. This representation explicitly decouples the system into two parts: 
\textbf{topology $T_t = (\mathcal{V}_t, \mathcal{E}_t)$}, where $\mathcal{V}_t$ is the set of agents (vertices) and $\mathcal{E}_t \subseteq \mathcal{V}_t \times \mathcal{V}_t$ is the set of directed edges (information channels).
In addition, we have \textbf{capability $\Phi_t = \{\phi_{v,t}\}_{v \in \mathcal{V}_t}$}, denoting the collection of capability states, where each $\phi_{v,t}$ encompasses an agent's specific prompt, contextual memory, and tool inventory. 
In our framework, a \textbf{Meta-LLM} $\mathcal{M}$ initializes $\mathcal{G}_0$ and orchestrates its subsequent evolution. The agents $v \in \mathcal{V}_t$ in the graph instantiate specific roles from a fixed pool $\mathcal{R}$ (\textit{e.g.,} {Planner, Searcher, Verifier}).


\begin{figure}
    \centering
    \includegraphics[width=\linewidth]{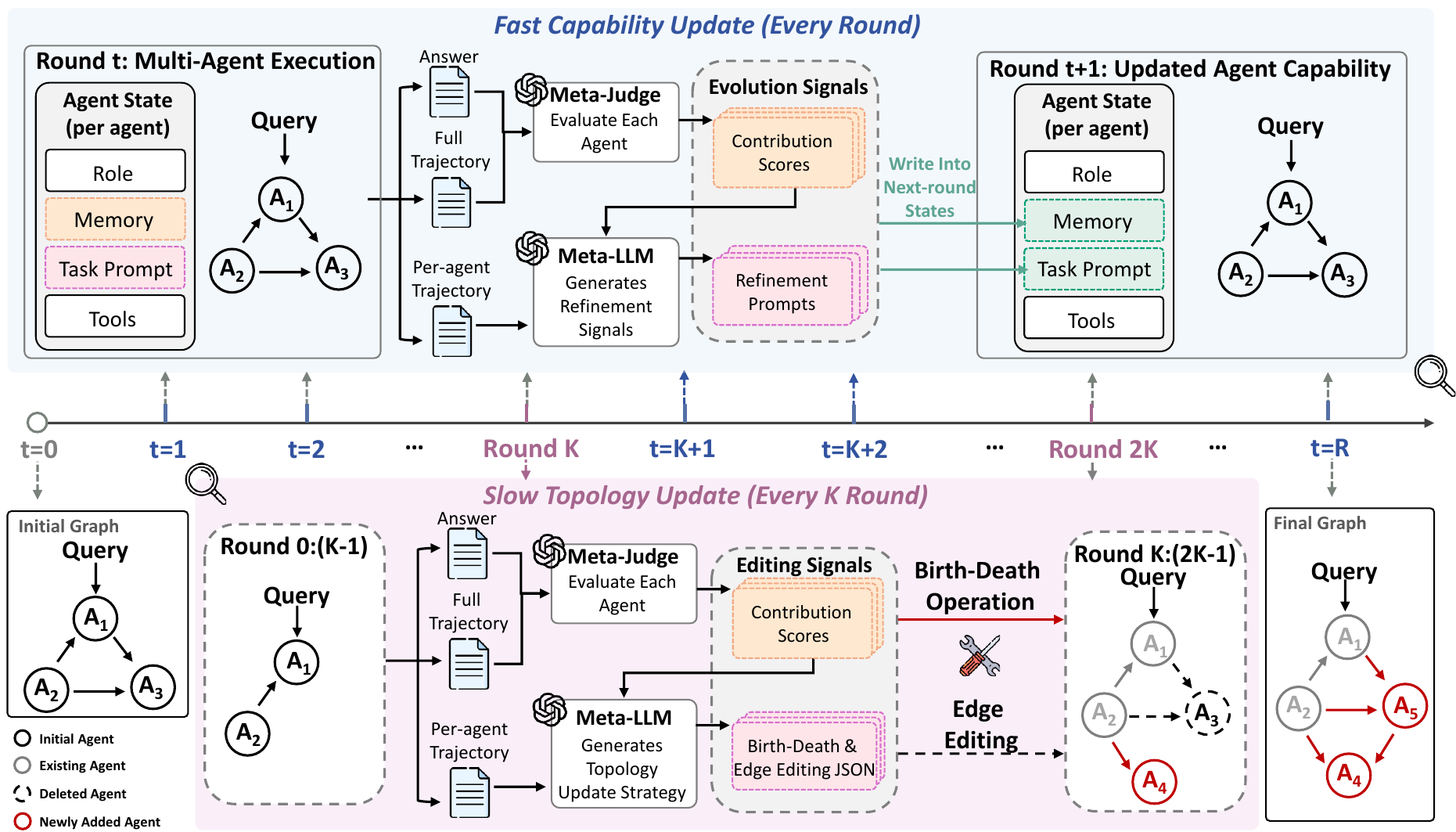}
    \caption{The overview framework of TacoMAS.}
    \label{fig:method}
\end{figure}





\subsection{Two-time-scale Dynamics}
\label{sec:method-twoscale}

The central design of \methodname\ is the asynchronous co-evolution of agent capabilities $\Phi$ and topology $T$ on two distinct time scales. This joint update process is formulated as:
\begin{equation}
\Phi_{t+1} = F^C(\Phi_t; T_t, \xi_t), \qquad T_{t+K} = F^T(T_t; \Phi_{t:t+K}, \xi_{t:t+K}),
\label{eq:twoscale}
\end{equation}
where $F^C$ and $F^T$ denote the capability and topology operators, respectively, and $K \geq 1$ is the slow-update interval. Specifically, the fast capability update $F^C$ occurs in every execution round. It allows agents to immediately incorporate feedback from the trajectory $\xi_t$ to adapt their reasoning patterns and tool-use strategies within the current topology. In contrast, the slow topology update $F^T$ modifies the communication topology only after $K$ rounds. This slower rhythm ensures that the topology remains stable for a sufficient duration, allowing agents to reach their performance ceiling under the given topology before the system considers a structural overhaul.

This two-time-scale design is essential to maintain the stability of the co-evolution process. If the topology $T$ changes as rapidly as individual capabilities $\Phi$, the refined strategies of agents may become obsolete due to sudden shifts in their information sources or collaborators. Such rapid structural changes can lead to systemic divergence. By decoupling these two processes, the fast dynamics effectively track a quasi-stationary equilibrium under a fixed architecture. The slow loop then optimizes the underlying graph topology based on the aggregated performance observed across multiple rounds. Consequently, the interval $K$ serves as a critical parameter to balance local strategy adaptation with global structural exploration.

\subsection{Fast Capability Loop $F^C$}
\label{sec:method-fast}

Within each execution round $t$, the fast capability loop optimizes the expertise of individual agents under a fixed topology $T_t$. Every agent $v \in \mathcal{V}_t$ executes its assigned role based on its current capability state $\phi_{v,t}$, which is instantiated through a combination of role-specific instructions and contextual memory. This process generates a \textit{per-agent trajectory} $r_{v,t}$, including reasoning steps, tool-use outcomes, and outgoing messages. The per-agent trajectory collectively forms the round's full execution trajectory $\xi_t = \{r_{v,t}\}_{v \in \mathcal{V}_t}$. 

\paragraph{Capability update via memory refinement.} 
In practice, the capability update $F^C$ is realized by a meta-judge and a meta-LLM acting as a diagnostic coach. After each round, the system generates evolution signals that are written back to the agent's state to update $\phi_{v,t}$, which includes two parts: 
1) \textit{Evaluation signals}: To ensure objective assessment, the meta-judge evaluates each agent's behavior based on the full trajectory $\xi_t$ to provide a numerical contribution score $c_{v,t}$ and a textual justification for the rating. 
2) \textit{Refinement signals}: To improve each agent's capability, the meta-LLM diagnoses the agent's specific \textit{per-agent trajectory} $r_{v,t}$ and the meta-judge's feedback $(c_{v,t}, \text{justification})$. 
It generates feedback identifying specific errors in $r_{v,t}$ and a concrete execution plan for the subsequent round. 
During the next round's initialization, these results are incorporated into the agent's contextual prompt, effectively refining its capability state $\phi$ via memory refinement (detailed prompts for meta-judge and meta-LLM can be found in App.~\ref{app:prompts}). 


\paragraph{Theoretical abstraction of capability evolution.} 
To analyze this process, we model the agents' capability evolution as a discrete replicator-style update \cite{hofbauer1998evolutionary} over the capability states. Intuitively, this mechanism acts as a ``selection pressure'' that reallocates computational influence toward higher-performing behaviors \cite{hofbauer1998evolutionary}.
\begin{equation}
\phi_{v,t+1} \propto \phi_{v,t} \exp\Bigl(\eta\bigl(c_{v,t} - \bar{c}_t\bigr)\Bigr), \quad \bar{c}_t = \tfrac{1}{|\mathcal{V}_t|}\sum_{u \in \mathcal{V}_t} c_{u,t}
\label{eq:fast}
\end{equation}
where $\bar{c}_t$ is the mean contribution and $\eta > 0$ controls the update strength \cite{hofbauer1998evolutionary}. This formulation captures the population-level effect of the agent's capability state updates: while the meta-LLM provides textual refinement for all agents, the reinforcement is biased such that high-contributing patterns are amplified and prioritized, while erroneous or marginal behaviors are effectively suppressed within the team's collective reasoning \cite{hofbauer1998evolutionary}.

\paragraph{Connecting theoretical abstraction to meta-LLM actions.} 
To ensure that these implementation-level actions are consistent with the replicator flow (Eq. (\ref{eq:fast})), we introduce the following assumption to justify that the meta-LLM effectively drives the system towards higher performance. 


\begin{assumption}[Replicator-bias of the meta-LLM update]
\label{ass:llm-bias}
There exists $\eta > 0$ and slack $\epsilon_{\textrm{noise}} \ge 0$ such that, for every fast round $t$:
\begin{equation}
\mathbb{E}\bigl[\bar{c}_{t+1} \mid \mathcal{H}_t\bigr] \ge \bar{c}_t + \eta\,\mathbb{V}\mathrm{ar}_v \bigl[c_{v,t}\bigr] - \eta\,\epsilon_{\textrm{noise}}
\label{eq:llm-bias}
\end{equation}
where $\bar{c}_t$ is the team mean contribution, and $\mathcal{H}_t = \{\xi_\tau, c_\tau\}_{\tau \le t}$ denotes the filtration of trajectories and scores up to round $t$.
\end{assumption}

This assumption implies that the meta-LLM's refinement acts as a \textit{Shahshahani gradient ascent} on the mean fitness, ensuring that the heuristic memory updates are statistically aligned with the formal replicator dynamics. Specifically, it guarantees that the textual modifications systematically improve the MAS performance (empirical justification is provided in App.~\ref{app:assump_verify}). 







\subsection{Slow Topology Loop $F^T$}
\label{sec:method-slow}

While the fast capability loop optimizes per-agent capability, the slow update $F^T$ reconfigures the MAS topology by modifying the sets of agents $\mathcal{V}$ and edges $\mathcal{E}$. 
After every $K$ rounds, the meta-LLM $\mathcal{M}$ proposes a structural delta $\Delta T = (\Delta\mathcal{V}, \Delta\mathcal{E})$ to resolve systemic bottlenecks that individual capability refinement cannot fix. 

\paragraph{Per-agent birth-death and edge edits.}
The structural delta $\Delta T$ is realized through two operations:
1) {Birth-Death:} A \textit{birth} introduces a new agent role to expand functional capacity, while a \textit{death} removes agents whose contribution scores $c_{v}$ remain consistently low. This process mimics discrete mutation by altering the system's ``population support'' to escape local optima. 
2) {Edge Reconfiguration:} $\mathcal{M}$ adds or removes communication channels to repair information flow. For instance, if a verifier lacks sufficient context, $\mathcal{M}$ may create a new edge from a high-contribution searcher to bridge the evidence gap. The two operations are implemented via textual prompt (see App.~\ref{app:prompts}).


\paragraph{Update stability.}
To maintain the stability of the two-time-scale dynamics, we introduce {edit budgets} on the structural update: 
\begin{equation}
T_{t+K} = T_t \oplus \Delta T, \quad |\Delta\mathcal{V}| \le B_{\mathcal{V}},\;\; |\Delta\mathcal{E}| \le B_{\mathcal{E}},
\label{eq:slow}
\end{equation}
where $B_{\mathcal{V}}$ and $B_{\mathcal{E}}$ represent the maximum allowed edits for agents and edges, respectively. This constraint prevents abrupt topological shifts from destabilizing the refined capability states $\phi$. By limiting structural volatility, we ensure that the progress gained through fast-loop evolution is preserved during reconfiguration.

\begin{algorithm}[htbp]
\small
\caption{\methodname\ Procedure }
\label{alg:tc-evo}
\begin{algorithmic}[1]
\Require query $q$; round cap $R$; slow interval $K$; edit budget $\mathbf{B}=(B_\V,B_\E)$; threshold $\tau$.
\State $\G_0 \leftarrow \mathcal{M}_{\textrm{init}}(q)$;\quad $t \leftarrow 0$.
\Repeat
    \For{$v \in \V_t$} \hfill \Comment{fast capability loop $F^C$}
        \State execute tool policy of $v$ under $(\phi_{v,t},\E_t)$, observe $r_{v,t}$
        \State $c_{v,t} \leftarrow \mathcal{J}(r_{v,t})$;\quad $\phi_{v,t+1} \leftarrow F^C(\phi_{v,t}; r_{v,t},c_{v,t})$ \hfill \eqref{eq:fast}
    \EndFor
    \State read answer $a_t$, score $s_t \leftarrow \mathcal{J}(a_t,q)$
    \If{$(t+1)\bmod K = 0$ \textbf{and} $s_t < \tau$} \hfill \Comment{slow topology loop $F^T$}
        \State $\Delta T, \{\Delta\phi_v\} \leftarrow \mathcal{M}_{\textrm{slow}}(T_t, \Phi_{t-K:t}, \{c_{v,\tau}\})$
        \State enforce $|\Delta \V| \le B_\V,\,|\Delta \E| \le B_\E$;\quad $T_{t+1} \leftarrow T_t \oplus \Delta T$ \hfill \eqref{eq:slow}
    \EndIf
    \State $t \leftarrow t + 1$
\Until{$s_t \ge \tau$ \textbf{or} $t \ge R$ \textbf{or} \textsc{stop}}
\State \Return $a_t$ from the sink agent of $\G_t$.
\end{algorithmic}
\end{algorithm}

\paragraph{Initialization and termination.}
The meta-LLM $\mathcal{M}$ seeds the initial graph $\mathcal{G}_0$ by selecting roles from the pool $\mathcal{R}$; we set $|\mathcal{V}_0| = 5$. The evolution process terminates when one of the following conditions is met: 1) the global score $s_t$ reaches the task-specific success threshold $\tau$; 2) the execution reaches the maximum round budget $R$; or 3) the meta-LLM issues a \textit{stop} signal upon detecting convergence in the agent trajectories. 

\section{Theoretical Analysis}
\label{sec:theory}

We provide a lightweight analysis of \methodname{} as a two-time-scale replicator--mutator process.  Full proofs are provided in App.~\ref{app:proofs}.

\subsection{Fast Loop as Replicator Dynamics}
\label{sec:theory-fast}

The fast capability update in Eq.~\eqref{eq:fast} has the standard form
of a discrete replicator update: behaviors with above-average
contribution are amplified, while below-average behaviors are
suppressed. Under a fixed topology $T_t$, this update approximates the
continuous replicator flow
\begin{equation}
\dot{\phi}_{v}
=
\phi_v\bigl(f_v-\bar f\bigr),
\label{eq:replicator}
\end{equation}
where $f_v$ denotes the expected contribution of agent $v$ and
$\bar f$ is the team-average contribution. This flow is a
Shahshahani-gradient ascent on mean fitness
\citep{akin1979geometry,hofbauer1998evolutionary}.

\begin{assumption}[Bounded contribution noise]
\label{ass:fitness-bdd}
The meta-judge contribution score $c_{v,t}$ is a bounded noisy estimate
of the expected contribution $f_v$, with noise bounded by $\epsilon$.
\end{assumption}

\begin{proposition}[Fast-loop improvement]
\label{prop:fast-monotone}
Under Assumption~\ref{ass:fitness-bdd}, one fast
update satisfies
\begin{equation}
\mathbb{E}\!\left[\bar f_{t+1}\mid \mathcal{H}_t\right]
\ge
\bar f_t - \eta\epsilon .
\end{equation}
Moreover, when contribution variance is nonzero, the expected update is
biased toward increasing the team-average contribution.
\end{proposition}

Proposition~\ref{prop:fast-monotone} formalizes the role of the
capability loop: it improves agents' local reasoning strategies under
the current communication structure. However, it cannot add new agents,
remove ineffective ones, or repair missing communication channels. Thus,
the fast loop may converge to a topology-dependent plateau.

\subsection{Slow Loop as Bounded Mutation}
\label{sec:theory-slow}

The slow topology update addresses this limitation. Every $K$ rounds,
the meta-LLM applies a bounded structural edit
$\Delta T=(\Delta\mathcal{V},\Delta\mathcal{E})$, as defined in
Eq.~\eqref{eq:slow}. Birth--death operations change the agent support,
while edge edits change the communication topology. These operations
act as mutation steps over the current multi-agent organization.

\begin{assumption}[Bounded and biased topology edits]
\label{ass:meta-bias}
Each slow update obeys the edit budgets in Eq.~\eqref{eq:slow}. In
addition, conditioned on the recent trajectory, the proposed edit
improves the best achievable team contribution under the topology with
probability $p>1/2$.
\end{assumption}

This assumption captures the intended behavior of the meta-LLM:
it is not required to always find a better topology, but its edits are
more likely to move the system toward a better communication topology
than away from it.

\subsection{Joint Two-Time-Scale Convergence}
\label{sec:theory-joint}

Combining the two loops yields a replicator-mutator process. The fast
replicator phase moves the agents toward a local performance plateau
under the current topology, and the slow mutation phase changes the
topology when this plateau is insufficient. Let $L(\Phi,T)$ denote the
distance to the set of locally stable high-performing configurations,
as defined in App. ~\ref{app:proofs}.

\begin{theorem}[Two-time-scale convergence]
\label{thm:joint}
Under Assumptions~\ref{ass:fitness-bdd}--\ref{ass:meta-bias}, there
exists $\gamma>0$ such that the joint update satisfies
\begin{equation}
\mathbb{E}\!\left[L(\Phi_{t+1},T_{t+1})\mid\mathcal{H}_t\right]
\le
(1-\gamma)L(\Phi_t,T_t)+\tilde{\epsilon},
\end{equation}
where $\tilde{\epsilon}$ collects contribution-score noise,
meta-LLM errors, and discretization slack.
\end{theorem}

Theorem~\ref{thm:joint} shows that the expected distance to the stable
configuration set contracts up to a noise-controlled neighborhood. In
particular, the fast loop alone can only optimize capabilities within a
fixed topology, while the slow loop provides the mutation needed to
escape topology-induced plateaus. This explains why jointly evolving
capabilities and topology is more effective than updating either side
alone.
\section{Experimental Results}
\label{sec:experiments-results}

\subsection{Experimental Setup}
\label{sec:experiments-setup}

\paragraph{Benchmarks.}
According to \citet{kim2025towards}, we evaluate on four benchmarks covering distinct reasoning regimes:
\texttt{finance}~\citep{bigeard2025finance} for retrieval-heavy
financial analysis over SEC filings; \texttt{browsecomp-plus}~\citep{chen2025browsecomp}
for entity disambiguation and multi-hop retrieval over a curated corpus;
\texttt{plancraft}~\citep{dagan2024plancraft} for Minecraft-style
crafting and feasibility planning; and \texttt{workbench}~\citep{styles2024workbench}
for realistic workplace workflows with tool use. Together, they test
retrieval, planning, numerical reasoning, and cross-tool coordination.

\paragraph{Baselines.} We compare \methodname{} with 20 multi-agent baselines, grouped by when
adaptation occurs. \textit{Fixed-topology} methods use manually
specified communication patterns and never evolve. \textit{Offline-evolved}
methods search or optimize workflows once before deployment and freeze
them at inference. \textit{Per-instance} methods generate a graph for
each query but keep it fixed while solving. \textit{Within-instance}
methods adapt during inference, but only along one axis: topology
(ChatDev-Puppeteer, SelfOrg) or capability (CORAL).
All baselines use the same LLM backend and dataset-specific tools unless
otherwise stated. Detailed baseline descriptions are in
App.~\ref{app:baselines}.

\paragraph{Metrics.}

We report accuracy ($\acc$), which is the mean rubric-judged accuracy over instances~\citep{kim2025towards}. Each instance may contain one or more evaluation criteria, including correctness and contradiction rubrics, and the instance score is the fraction of rubric items passed.

\paragraph{Detailed settings.}
All agents use Gemini-2.5-flash-lite as the base LLM, while GPT-4o-mini is used as the rubric judge. For \methodname{}, the meta-LLM
defaults to Gemini-2.5-pro. Unless otherwise stated, \methodname{} runs for at most $R=10$ fast rounds with slow topology updates every $K=2$ rounds, agent cap $|\V|_{\max}=20$, at most two birth/death pairs, and at most four edge edits per slow update.
The initial graph is the same 5-role centralized template used by the fixed-topology baseline, consisting of planner, searcher, calculator,
verifier, and reflector agents, so improvements are attributable to evolution.

\begin{table*}[t]
\centering
\scriptsize
\setlength{\tabcolsep}{4.0pt}
\renewcommand{\arraystretch}{1.05}
\caption{
Main accuracy results. Baselines are grouped by their evolution time
scale. Best accuracy in each dataset is in bold. The best baseline in
each dataset is underlined, and the last row reports the absolute
accuracy improvement of \methodname{} over the best baseline.
}
\label{tab:main-results}
\begin{tabular}{lcccc}
\toprule
Method
& finance & browsecomp & plancraft & workbench \\
\midrule

\multicolumn{5}{l}{\textbf{Offline-evolved}} \\
MetaGPT~\citep{hong2023metagpt}             & 0.348 & 0.550 & 0.668 & 0.473 \\
AFlow~\citep{zhang2024aflow}                & 0.343 & 0.633 & 0.717 & 0.478 \\
AgentSquare~\citep{shang2024agentsquare}    & 0.286 & 0.625 & 0.574 & 0.440 \\
EvoAgentX~\citep{wang2025evoagentx}         & 0.313 & 0.591 & 0.763 & 0.419 \\
ADAS~\citep{hu2024automated}                & 0.294 & 0.595 & 0.497 & 0.645 \\
\midrule

\multicolumn{5}{l}{\textbf{Per-instance graph design}} \\
AgentVerse~\citep{chen2023agentverse}       & 0.301 & 0.625 & 0.653 & 0.523 \\
ARG-Designer~\citep{li2026assemble}         & 0.295 & 0.621 & 0.619 & 0.595 \\
MaAS~\citep{zhang2025multi}                 & 0.284 & 0.620 & 0.492 & 0.538 \\
MetaAgent~\citep{zhang2025metaagent}        & 0.349 & 0.601 & 0.560 & 0.469 \\
SwarmAgentic~\citep{zhang2025swarmagentic}  & 0.338 & 0.570 & \underline{0.812} & \underline{0.651} \\
MetaGen~\citep{wang2026metagen}             & 0.419 & 0.513 & 0.479 & 0.492 \\
EvolveRouter~\citep{huang2026evolverouter}  & 0.332 & 0.170 & 0.320 & 0.290 \\
\midrule

\multicolumn{5}{l}{\textbf{Fixed-topology}~\citep{kim2025towards}} \\
SAS                                         & \underline{0.539} & 0.200 & 0.530 & 0.347 \\
MAS-Independent                             & 0.529 & 0.090 & 0.710 & 0.416 \\
MAS-Decentralized                           & 0.445 & 0.240 & 0.600 & 0.446 \\
MAS-Centralized                             & 0.500 & 0.270 & 0.560 & 0.386 \\
MAS-Hybrid                                  & 0.537 & 0.260 & 0.540 & 0.386 \\
\midrule

\multicolumn{5}{l}{\textbf{Within-instance evolution}} \\
ChatDev-Puppeteer~\citep{qian2024chatdev}   & 0.340 & 0.603 & 0.553 & 0.441 \\
SelfOrg~\citep{tastan2026stochastic}        & 0.377 & \underline{0.688} & 0.712 & 0.441 \\
CORAL~\citep{qu2026coral}                   & 0.409 & 0.505 & 0.458 & 0.511 \\
\midrule

\multicolumn{5}{l}{\textbf{Topology-capability co-evolution}} \\
\textbf{\methodname{}}
& \textbf{0.767}
& \textbf{0.745}
& \textbf{0.887}
& \textbf{0.824} \\
\quad Improvement
& \textbf{+22.8}
& \textbf{+5.7}
& \textbf{+7.5}
& \textbf{+17.3} \\
\bottomrule
\end{tabular}
\end{table*}

\subsection{Main Results}
\label{sec:main-results}

Table~\ref{tab:main-results} compares \methodname{} with 20 multi-agent
baselines on four benchmarks. \methodname{} achieves the best accuracy
across all datasets, outperforming offline-evolved workflows,
per-instance graph design methods, and fixed-topology MAS baselines. These results suggest that pre-optimized workflows, topology-only adaptation, and manually fixed collaboration patterns are often insufficient for diverse test-time tasks. Compared with within-instance evolution methods, \methodname{} also achieves stronger performance, showing that jointly
evolving agent capabilities and communication topology is especially
useful for tasks whose optimal topology and agent skills change across
different stages of inference.

\subsection{Analysis of \methodname{}}
\label{sec:exp_analysis}

\begin{figure}[t]
\centering
\includegraphics[width=0.9\linewidth]{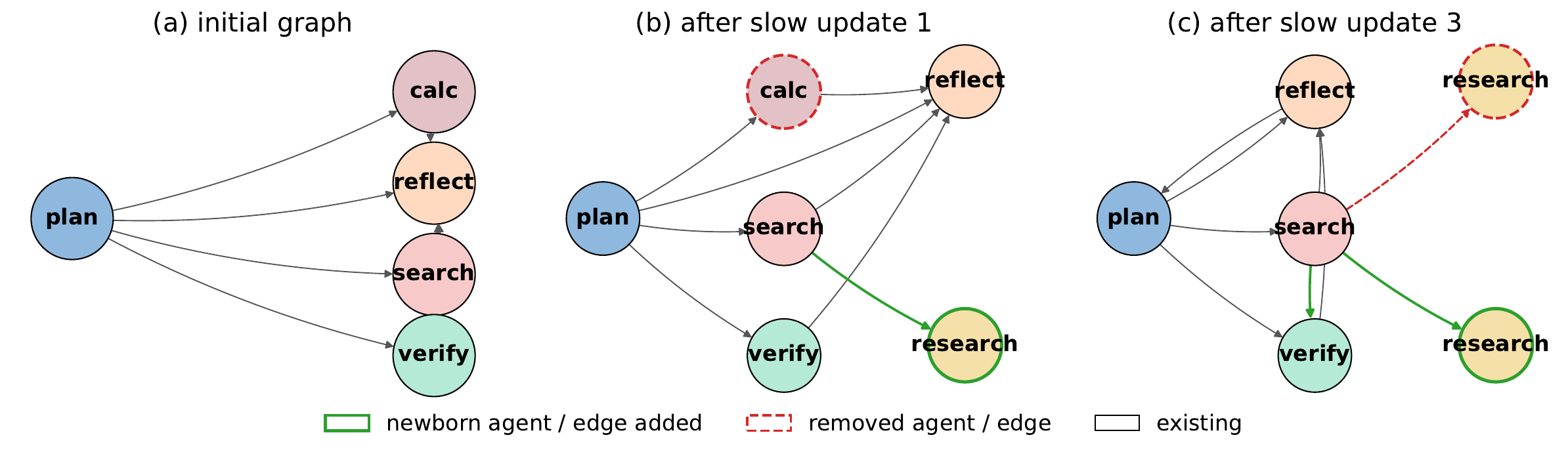}
\caption{
Representative evolution trace on a \texttt{finance} instance. The
initial graph evolves into a task-specific search-research-verify pipeline. For example, a ``link research'' is deleted while a ``data research'' is newly added after 3 rounds of slow update. 
}
\label{fig:evolution-trace}
\end{figure}

\paragraph{Evolution trajectory.}
Figure~\ref{fig:evolution-trace} shows that \methodname{} changes the
team organization during inference rather than only increasing compute.
The graph removes unhelpful roles, introduces missing capabilities, and
strengthens useful communication paths. This indicates that the
meta-LLM learns an instance-specific division of labor instead
of applying a fixed template (more cases in App.~\ref{app:more-evolution-traces}). 

\begin{figure}[t]
    \centering
    \includegraphics[width=0.9\linewidth]{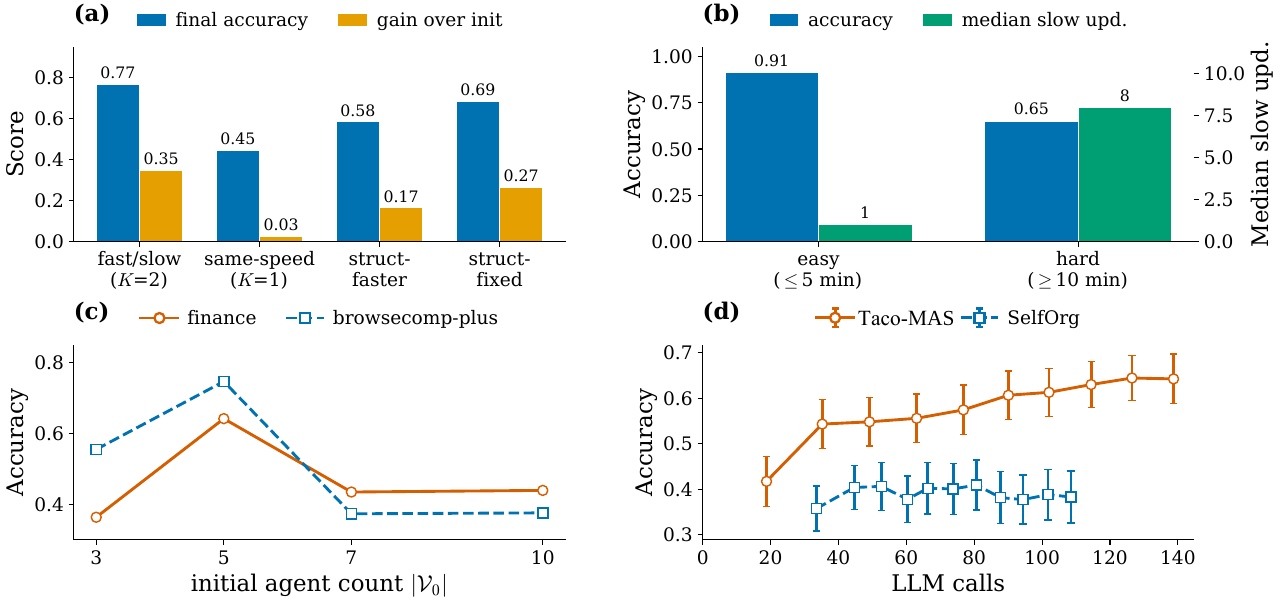}
    \caption{
    \textbf{(a)} Accuracy and gain under different fast/slow update schedules on \texttt{finance} benchmark.
    \textbf{(b)} Accuracy and median slow-update count by expert-time tier on \texttt{finance} benchmark.
    \textbf{(c)} Accuracy under different initial agent counts.
    \textbf{(d)} Accuracy \textit{w.r.t.} LLM calls under different models.
    }
    \label{fig:exp_analysis}
\end{figure}

\paragraph{Fast/slow update schedule.}
Figure~\ref{fig:exp_analysis}(a) shows that the fast/slow separation is
important. Updating topology too frequently makes the trajectory
unstable, while freezing topology prevents full specialization. The
default schedule works best because capabilities adapt quickly to local
failures, while topology changes more slowly to preserve coordination.

\paragraph{Difficulty-adaptive evolution.}
On the \texttt{finance} benchmark, we group instances by the provided
expert-time annotation and report the number of slow updates in
Figure~\ref{fig:exp_analysis}(b). \methodname{} spends more slow updates on
higher-expert-time instances, although this annotation is never observed by
the model. This supports that its extra computation is adaptively allocated
from trajectory feedback to achieve better performance in complex tasks, rather than uniformly increasing reasoning depth.

\paragraph{Computational cost analysis.}
We measure computational cost using $\calls$, defined as the mean number of LLM calls per instance, and compare \methodname{} with SelfOrg, the strongest within-instance evolution baseline, as shown in Figure~\ref{fig:exp_analysis}(d). The results show that increasing inference cost alone does not guarantee better performance:
SelfOrg requires many LLM calls but quickly reaches a performance plateau. In contrast, \methodname{} continues to improve as more evolution rounds are added, indicating that its additional inference budget is effectively converted into accuracy gains. This suggests that
\methodname{} has the potential to exhibit inference-time scaling
behavior, where performance can improve with increased test-time
computation.

\subsection{Robustness Analysis}
\label{sec:robust_analysis}

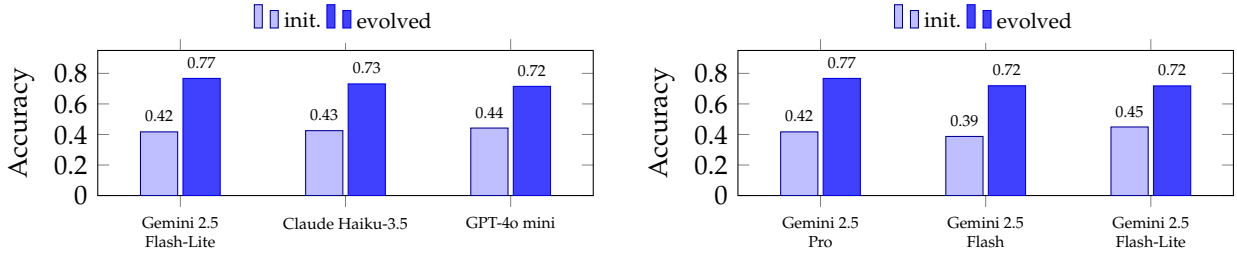
\begin{figure}[t]
\centering
\begin{subfigure}[t]{0.49\linewidth}
\centering
\begin{tikzpicture}
\begin{axis}[
  ybar,
  bar width=14pt,
  width=\linewidth,
  height=3.5cm,
  ymin=0, ymax=0.95,
  ylabel={Accuracy},
  xtick={1,2,3},
  xticklabels={{Gemini 2.5 \\ Flash-Lite}, Claude Haiku-3.5, GPT-4o mini},
  xticklabel style={font=\tiny, align=center},
  enlarge x limits=0.25,
  nodes near coords,
  nodes near coords style={font=\tiny},
  tick label style={font=\small},
  ylabel style={font=\small},
  legend style={at={(0.5,1.05)}, anchor=south, legend columns=2, draw=none, font=\scriptsize},
]
\addplot[fill=blue!25, draw=blue!60!black] coordinates {(1,0.417) (2,0.425) (3,0.442)};
\addplot[fill=blue!75, draw=blue!90!black] coordinates {(1,0.767) (2,0.731) (3,0.715)};
\legend{init., evolved}
\end{axis}
\end{tikzpicture}
\label{fig:bar-backbone}
\end{subfigure}
\hfill
\begin{subfigure}[t]{0.49\linewidth}
\centering
\begin{tikzpicture}
\begin{axis}[
  ybar,
  bar width=14pt,
  width=\linewidth,
  height=3.5cm,
  ymin=0, ymax=0.95,
  ylabel={Accuracy},
  xtick={1,2,3},
  xticklabels={{Gemini 2.5 \\ Pro}, {Gemini 2.5 \\ Flash}, {Gemini 2.5 \\ Flash-Lite}},
  xticklabel style={font=\tiny, align=center},
  enlarge x limits=0.25,
  nodes near coords,
  nodes near coords style={font=\tiny},
  tick label style={font=\small},
  ylabel style={font=\small},
  legend style={at={(0.5,1.05)}, anchor=south, legend columns=2, draw=none, font=\scriptsize},
]
\addplot[fill=blue!25, draw=blue!60!black] coordinates {(1,0.417) (2,0.387) (3,0.449)};
\addplot[fill=blue!75, draw=blue!90!black] coordinates {(1,0.767) (2,0.719) (3,0.718)};
\legend{init., evolved}
\end{axis}
\end{tikzpicture}
\label{fig:bar-meta}
\end{subfigure}
\caption{
Pre- vs post-evolution accuracy under (1) left: different
agent LLM backbones and (2) right: different meta-LLM backbones.
}
\label{fig:bar-init-final}
\end{figure}

\paragraph{Initial agent count.}
Figures~\ref{fig:exp_analysis}(c) show that performance is
non-monotonic in the initial team size. Too few agents limit role
diversity, while too many agents increase coordination complexity and
dilute the effect of each bounded topology edit. Meanwhile, inference
cost grows with team size. This suggests that a moderate initial team
offers the best balance between role diversity, coordination stability,
and compute.

\paragraph{Backbone robustness.}
Figure~\ref{fig:bar-init-final} shows that evolution improves
performance across both agent backbones and meta-LLM backbones.
Although different models produce different initial graphs and outputs,
the post-evolution gains remain consistent. This suggests that the main
benefit comes from the co-evolution process rather than a specific LLM
backbone.
\section{Conclusion}
\label{sec:conclusion}

We present \methodname, a test-time co-evolution framework for LLM-based multi-agent systems, which formulates inference as a two-time-scale online adaptation process that jointly refines agent capabilities rapidly and communication topology slowly within each task instance. 
Theoretically, we connect the framework to a replicator-mutator process and show that the joint dynamics can contract toward a task-conditioned stable region.
Empirically, \methodname{} achieves the best accuracy on all four benchmarks. Further analyses show that the fast/slow update schedule is crucial, and the amount of evolution is adaptively adjusted according to task difficulty.
Overall, our results suggest that inference-time computation in multi-agent LLM systems should not be conceived as a static forward pass, but a temporal process of co-evolution. Limitations are in the App.~\ref{sec:limitation}.


\bibliography{ref}

\newpage
\appendix
\section{Full Proofs}
\label{app:proofs}

We provide full proofs for the replicator--mutator analysis in
Sec.~\ref{sec:theory}.  Notation follows
Sec.~\ref{sec:method-formal}: at execution round $t$ the agent graph
is $\mathcal{G}_t = (T_t, \Phi_t)$ with topology
$T_t = (\mathcal{V}_t, \mathcal{E}_t)$ and capability collection
$\Phi_t = \{\phi_{v,t}\}_{v \in \mathcal{V}_t}$;\, the meta-judge
$\mathcal{J}$ produces contribution scores $c_{v,t}$ with team mean
$\bar c_t = m_t = \tfrac{1}{|\mathcal{V}_t|}\sum_u c_{u,t}$;\, $\eta>0$
is the replicator step;\, and $K \ge 1$ is the slow-update interval.
We assume $|\mathcal{V}_t| \le N_{\max}$ and
$|\mathcal{E}_t| \le N_{\max}^2$ throughout, so all relevant random
variables live on a finite state space.

\subsection{Population-Frequency View}
\label{app:freq-view}

Eq.~\eqref{eq:fast} treats $\phi_{v,t}$ as a positive scalar weight
that aggregates the agent's prompt, memory and tool inventory into a
single multiplicative influence.  For analysis we work with the
induced \emph{role-frequency vector}
\begin{equation}
\pi_{v,t}
\;=\;
\frac{\phi_{v,t}}{\sum_{u \in \mathcal{V}_t}\phi_{u,t}}
\;\in\; \Delta^{|\mathcal{V}_t|},
\label{eq:pi-from-phi}
\end{equation}
and the \emph{population-game expected fitness} of role $v$ at state
$(\pi, T)$:
\begin{equation}
f_v(\pi, T;q)
\;=\;
\mathbb{E}\!\left[c_{v,t} \,\bigm|\, \pi_t = \pi,\, T_t = T,\, q\right],
\quad
\bar f(\pi, T)
\;=\;
\sum_{v}\pi_v\, f_v(\pi, T; q).
\end{equation}
By Assumption~\ref{ass:fitness-bdd} the noise
$\zeta_{v,t} := c_{v,t} - f_v(\pi_t, T_t; q)$ is bounded,
$|\zeta_{v,t}| \le \epsilon$ a.s., and $f_v$ is bounded and Lipschitz
in $(\pi, T)$.

\subsection{Replicator Dynamics under a Frozen Topology}
\label{app:replicator-lyap}

Fix a topology $T$.  The continuous-time replicator equation
\begin{equation}
\dot{\pi}_v
\;=\;
\pi_v\bigl(f_v(\pi, T) - \bar f(\pi, T)\bigr),
\qquad v \in \mathcal{V},
\label{eq:cont-replicator}
\end{equation}
is a Shahshahani gradient ascent on the team-average fitness
$\bar f$ \citep{akin1979geometry,hofbauer1998evolutionary}: along
the flow,
\begin{equation}
\frac{d}{dt}\bar f(\pi(t), T)
\;=\;
\sum_{v}\pi_v(f_v - \bar f)^{2}
\;=\;
\mathrm{Var}_{\pi}(f)
\;\ge\; 0,
\label{eq:lyap-cont}
\end{equation}
so $\bar f$ is a Lyapunov function for the continuous flow.

\paragraph{Discrete-time approximation.}
Substituting~\eqref{eq:pi-from-phi} into~\eqref{eq:fast} and using the
fact that the empirical mean $\bar c_t$ in the exponent cancels in the
ratio gives the equivalent normalized form
\begin{equation}
\pi_{v,t+1}
\;=\;
\frac{\pi_{v,t}\exp(\eta\, c_{v,t})}
     {\sum_{u}\pi_{u,t}\exp(\eta\, c_{u,t})}.
\label{eq:disc-replicator}
\end{equation}
A second-order Taylor expansion of $\exp$ in $\eta$ yields, with
$\langle c \rangle_{\pi_t} := \sum_u \pi_{u,t} c_{u,t}$,
\begin{equation}
\pi_{v,t+1}
\;=\;
\pi_{v,t}
\;+\;
\eta\,\pi_{v,t}\bigl(c_{v,t} - \langle c \rangle_{\pi_t}\bigr)
\;+\;
O(\eta^{2}).
\label{eq:disc-step}
\end{equation}
Substituting $c_{v,t} = f_v(\pi_t, T) + \zeta_{v,t}$, taking
expectation conditioned on $\mathcal{H}_t$ and using
$|\zeta_{v,t}| \le \epsilon$ together with the Lipschitz continuity
of $\bar f$, we obtain the discrete-time fitness ascent
\begin{equation}
\mathbb{E}\!\left[\bar f(\pi_{t+1}, T)\,\bigm|\,\mathcal{H}_t\right]
\;\ge\;
\bar f(\pi_t, T)
\;+\;
\eta\,\mathrm{Var}_{\pi_t}(f)
\;-\;
\eta\,\epsilon
\;-\;
O(\eta^{2}).
\label{eq:fast-fitness}
\end{equation}
Dropping the nonnegative variance term yields the weaker but
sufficient monotone bound
\begin{equation}
\mathbb{E}\!\left[\bar f(\pi_{t+1}, T)\,\bigm|\,\mathcal{H}_t\right]
\;\ge\;
\bar f(\pi_t, T) - \eta\,\epsilon - O(\eta^{2}).
\label{eq:fast-fitness-weak}
\end{equation}

\subsection{Lyapunov Function Construction}
\label{app:lyap-fn}

Fix the query $q$.  Let
\[
\mathcal{A}(q)
\;=\;
\bigl\{(\pi^{\ast}, T^{\ast}) :
\pi^{\ast}\text{ is a local maximizer of } \bar f(\,\cdot\,, T^{\ast})\bigr\}
\]
denote the set of locally optimal configurations under query $q$.
Define the topology distance
\begin{equation}
d(T, \mathcal{A})
\;=\;
\min_{(\pi^{\ast}, T^{\ast}) \in \mathcal{A}}
\bigl(|\mathcal{V} \triangle \mathcal{V}^{\ast}|
+ |\mathcal{E} \triangle \mathcal{E}^{\ast}|\bigr),
\end{equation}
the per-topology fitness ceiling
\begin{equation}
\bar f^{\ast}(T)
\;=\;
\max_{\pi'} \bar f(\pi', T),
\end{equation}
and the joint Lyapunov function
\begin{equation}
L(\Phi, T)
\;=\;
d(T, \mathcal{A})
\;+\;
\eta\bigl(\bar f^{\ast}(T) - \bar f(\pi, T)\bigr),
\label{eq:lyap-def}
\end{equation}
where $\pi$ is the role frequency induced by $\Phi$ via
Eq.~\eqref{eq:pi-from-phi}.  Both summands are bounded:
$d(T, \mathcal{A}) \le 2(N_{\max} + N_{\max}^{2})$ since
$|\mathcal{V}|, |\mathcal{E}|$ are bounded, and
$\bar f^{\ast}(T) - \bar f(\pi, T) \in [0, M]$ for some $M < \infty$
by Assumption~\ref{ass:fitness-bdd}.  Hence $L$ is bounded.

\subsection{Proof of Proposition~\ref{prop:fast-monotone}}

\begin{proof}
The topology is fixed inside one fast step, so~\eqref{eq:fast-fitness-weak}
applied at $T = T_t$ gives
\[
\mathbb{E}\!\left[\bar f_{t+1}\,\bigm|\,\mathcal{H}_t\right]
\;\ge\;
\bar f_t - \eta\,\epsilon - O(\eta^{2}),
\]
which yields the stated bound for $\eta$ small enough that
$O(\eta^{2}) \le \eta\epsilon$ (the regime of interest).
The ``biased toward improvement when contribution variance is nonzero''
clause follows from the stronger~\eqref{eq:fast-fitness}: the
$\eta\,\mathrm{Var}_{\pi_t}(f)$ term is strictly positive whenever the
expected fitnesses $\{f_v\}$ differ across the active agents.
\end{proof}

\subsection{Proof of Theorem~\ref{thm:joint}}

\begin{proof}
Consider one slow-update cycle: $K$ fast steps followed by one
topology update.  Write
$\pi^{\flat}_{t} := \pi_{t+K^{-}}$ for the role frequency just before
the slow update; the topology is unchanged during the fast phase, so
$T^{\flat}_t = T_t$.

\paragraph{Fast phase.}
Iterating~\eqref{eq:fast-fitness} for $K$ steps and using
$\mathrm{Var}_{\pi_t}(f) \ge 0$ throughout,
\begin{equation}
\mathbb{E}\!\left[\bar f^{\ast}(T_t) - \bar f(\pi^{\flat}_t, T_t)
\,\bigm|\, \mathcal{H}_t\right]
\;\le\;
\bigl(\bar f^{\ast}(T_t) - \bar f(\pi_t, T_t)\bigr)
\;+\;
K\eta\epsilon + O(K\eta^{2}).
\label{eq:fast-gap-bd}
\end{equation}
Because $T$ is fixed during the fast phase,
$d(T^{\flat}_t, \mathcal{A}) = d(T_t, \mathcal{A})$, so the Lyapunov
update over the fast phase satisfies
\begin{equation}
\mathbb{E}\!\left[L(\Phi^{\flat}_t, T_t)\,\bigm|\,\mathcal{H}_t\right]
\;\le\;
L(\Phi_t, T_t)
\;+\;
K\eta^{2}\epsilon
\;+\;
O(K\eta^{3}).
\label{eq:fast-Lbound}
\end{equation}

\paragraph{Slow phase.}
By Assumption~\ref{ass:meta-bias} the slow update obeys the edit
budget $|\Delta\mathcal{V}|\le B_{\mathcal{V}}$,
$|\Delta\mathcal{E}|\le B_{\mathcal{E}}$ (Eq.~\ref{eq:slow}), and
each edit improves the best achievable team contribution under the
topology with probability $p > 1/2$.  Each edit changes
$d(T, \mathcal{A})$ by at most $1$, so a standard biased
random-walk argument gives, with $\gamma := 2p - 1 > 0$,
\begin{equation}
\mathbb{E}\!\left[d(T_{t+K}, \mathcal{A})\,\bigm|\,\mathcal{H}_{t+K^{-}}\right]
\;\le\;
(1 - \gamma)\, d(T_t, \mathcal{A})
\;+\;
\epsilon_{\mathrm{meta}},
\label{eq:slow-d-bound}
\end{equation}
where $\epsilon_{\mathrm{meta}} \le B_{\mathcal{V}} + B_{\mathcal{E}}$
absorbs the bounded slack from edits that fail to decrease
$d(T,\mathcal{A})$.  By Lipschitz continuity of $\bar f^{\ast}$ in
$T$ on the finite topology lattice (a consequence of bounded $f$),
there is a constant $C > 0$ such that
\begin{equation}
\bigl|\bar f^{\ast}(T_{t+K}) - \bar f^{\ast}(T_t)\bigr|
\;\le\;
C\,(B_{\mathcal{V}} + B_{\mathcal{E}}),
\label{eq:f-star-lip}
\end{equation}
and similarly the projection of $\pi^{\flat}_t$ onto the new topology
$T_{t+K}$ changes $\bar f(\pi, T)$ by at most a constant multiple of
$B_{\mathcal{V}} + B_{\mathcal{E}}$.

\paragraph{Combined cycle bound.}
Combining~\eqref{eq:fast-Lbound}, \eqref{eq:slow-d-bound},
and~\eqref{eq:f-star-lip},
\begin{equation}
\mathbb{E}\!\left[L(\Phi_{t+K}, T_{t+K})\,\bigm|\,\mathcal{H}_t\right]
\;\le\;
(1 - \gamma)\, L(\Phi_t, T_t)
\;+\;
\tilde{\epsilon},
\label{eq:joint-contraction}
\end{equation}
where the noise term collects the contribution-score noise, the
meta-controller slack, the welfare-gap drift across the topology
edit, and the discretization error:
\[
\tilde{\epsilon}
\;=\;
\epsilon_{\mathrm{meta}}
\;+\;
K\eta^{2}\epsilon
\;+\;
\eta\,C\,(B_{\mathcal{V}} + B_{\mathcal{E}})
\;+\;
\gamma\,\eta\,M
\;+\;
O(K\eta^{3}),
\]
with $M$ the boundedness constant of the welfare gap from
\S\ref{app:lyap-fn}.

\paragraph{Iteration and convergence rate.}
Unrolling~\eqref{eq:joint-contraction} for $N$ slow cycles,
\begin{equation}
\mathbb{E}\!\left[L(\Phi_{NK}, T_{NK})\right]
\;\le\;
(1 - \gamma)^{N}\, L(\Phi_0, T_0)
\;+\;
\tilde{\epsilon}\sum_{i=0}^{N-1}(1-\gamma)^{i}
\;\le\;
(1 - \gamma)^{N}\, L_0
\;+\;
\frac{\tilde{\epsilon}}{\gamma}.
\label{eq:joint-iter}
\end{equation}
To reach $\eta\epsilon$-accuracy on the transient term we need
$(1-\gamma)^{N}L_0 \le \eta\epsilon$.  Using
$\log(1-\gamma) \le -\gamma$ this yields
\begin{equation}
N
\;=\;
O\!\left(\frac{1}{\gamma}\,\log\frac{L_0}{\eta\,\epsilon}\right).
\end{equation}
Recasting~\eqref{eq:joint-contraction} as the per-step contraction
stated in Theorem~\ref{thm:joint} (with the slow update absorbed into
a single round increment) completes the proof.
\end{proof}

\subsection{Two-Time-Scale Justification}

The above analysis is consistent with the standard two-time-scale
stochastic approximation framework
\citep{borkar1997stochastic,kushneryin2003}: for $K$ sufficiently
large, the fast dynamics tracks a quasi-stationary distribution
under the frozen topology before each slow update, and the slow
update operates on the time-averaged behaviour of these fast
trajectories.  The fast-phase Lyapunov increment
in~\eqref{eq:fast-Lbound} formalizes this quasi-stationarity in our
setting: each fast step changes $L$ by at most $O(\eta^{2}\epsilon)$
in expectation, so the cumulative drift over $K$ steps remains
$O(K\eta^{2}\epsilon) \to 0$ in the joint limit
$\eta \to 0$, $K \to \infty$, $K\eta = \mathrm{const}$.

\section{Limitations}
\label{sec:limitation}
Our current evolution trace is driven by a single meta-controller LLM,
which observes the full execution trajectory and proposes both capability
and topology updates. This design keeps the evolution process simple and
centralized, but it may become a bottleneck for very long or highly
decomposable instances. When a task naturally splits into multiple
sub-task clusters, a single controller may have to summarize too much
local information and may miss fine-grained coordination failures.
One possible extension is to use parallel meta-controllers, each
responsible for a sub-task cluster or a local region of the agent graph,
with a higher-level controller coordinating their proposed edits. Such a
hierarchical design may improve scalability while preserving global
consistency.

Our current implementation also clears the scratch memory $M_t$ between
instances. This avoids leaking task-specific content across queries and
keeps each test instance independent, but it prevents the system from
reusing useful structural discoveries. For example, the system may
repeatedly rediscover similar role decompositions, verifier--searcher
communication patterns, or tool-use strategies across related tasks. A
careful cross-instance memory design could store reusable structural
knowledge without retaining sensitive or instance-specific content.
Developing such memory mechanisms, together with safeguards against
spurious transfer and privacy leakage, is an important direction for
future work.

\section{Additional Experimental Details and Results}
\label{app:additional}

\subsection{Baseline Details}
\label{app:baselines}

We compare with 20 baselines grouped by adaptation time scale. All methods use the same dataset-specific tools and base LLM unless stated otherwise.

\paragraph{Fixed-topology~\citep{kim2025towards}.}
SAS uses one tool-augmented agent. MAS-Independent runs multiple agents independently and aggregates by voting. MAS-Decentralized uses a complete communication graph. MAS-Centralized uses a leader-worker structure. MAS-Hybrid uses a leader with worker sub-clusters.

\paragraph{Offline-evolved.}
These methods optimize workflow or agent design before deployment and freeze it at inference: MetaGPT~\citep{hong2023metagpt}, AFlow~\citep{zhang2024aflow}, AgentSquare~\citep{shang2024agentsquare}, EvoAgentx~\citep{wang2025evoagentx}, and ADAS~\citep{hu2024automated}.

\paragraph{Per-instance graph design.}
These methods generate or select a graph for each query, then keep it fixed during execution: AgentVerse~\citep{chen2023agentverse}, ARG-Designer~\citep{li2026assemble}, MaAS~\citep{zhang2025multi}, MetaAgent~\citep{zhang2025metaagent}, SwarmAgentic~\citep{zhang2025swarmagentic}, MetaGen~\citep{wang2026metagen}, and EvolveRouter~\citep{huang2026evolverouter}.

\paragraph{Within-instance evolution.}
These methods adapt during inference but only along one axis. ChatDev-Puppeteer~\citep{qian2024chatdev} selects among fixed personas. SelfOrg~\citep{tastan2026stochastic} rewires a communication DAG while keeping prompts fixed. CORAL~\citep{qu2026coral} updates memory and skills while leaving topology implicit. \methodname{} jointly evolves both topology and capabilities.

\subsection{Evolution Traces on Other Datasets}
\label{app:more-evolution-traces}

Figures~\ref{fig:app-evo-browsecomp}--\ref{fig:app-evo-workbench} show representative traces on the remaining datasets. Each trace shows the initial centralized graph, the graph after one slow update, and the graph after three slow updates or convergence. Green marks additions; red dashed marks removals.

\begin{figure}[h]
\centering
\includegraphics[width=\linewidth]{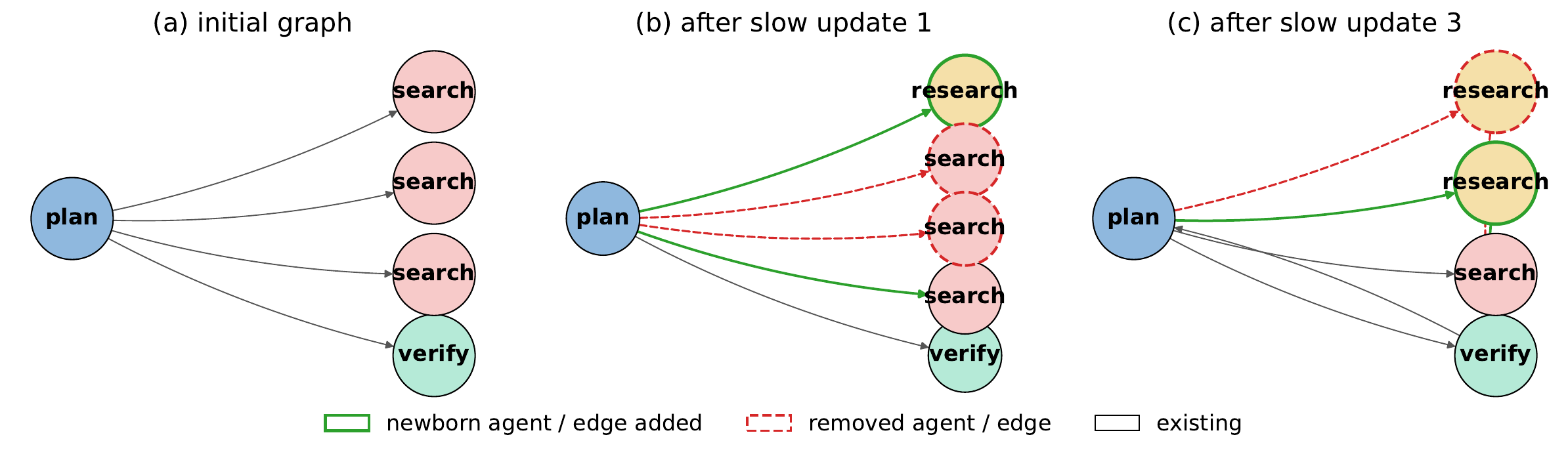}
\caption{\texttt{browsecomp-plus} trace.}
\label{fig:app-evo-browsecomp}
\end{figure}

\begin{figure}[h]
\centering
\includegraphics[width=\linewidth]{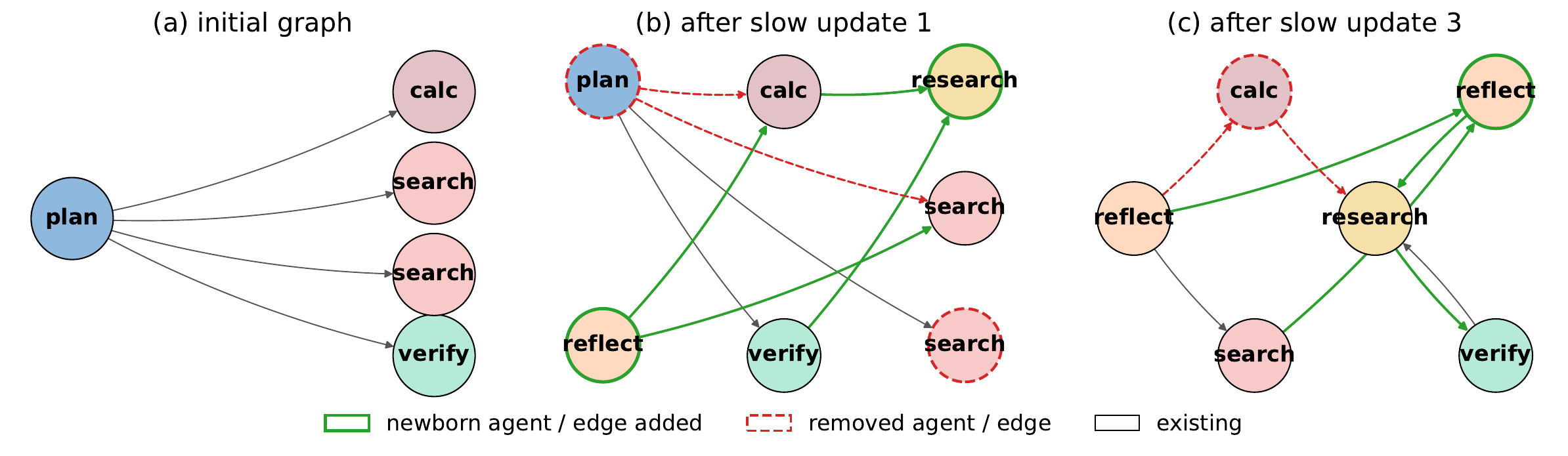}
\caption{\texttt{plancraft} trace. }
\label{fig:app-evo-plancraft}
\end{figure}

\begin{figure}[h]
\centering
\includegraphics[width=\linewidth]{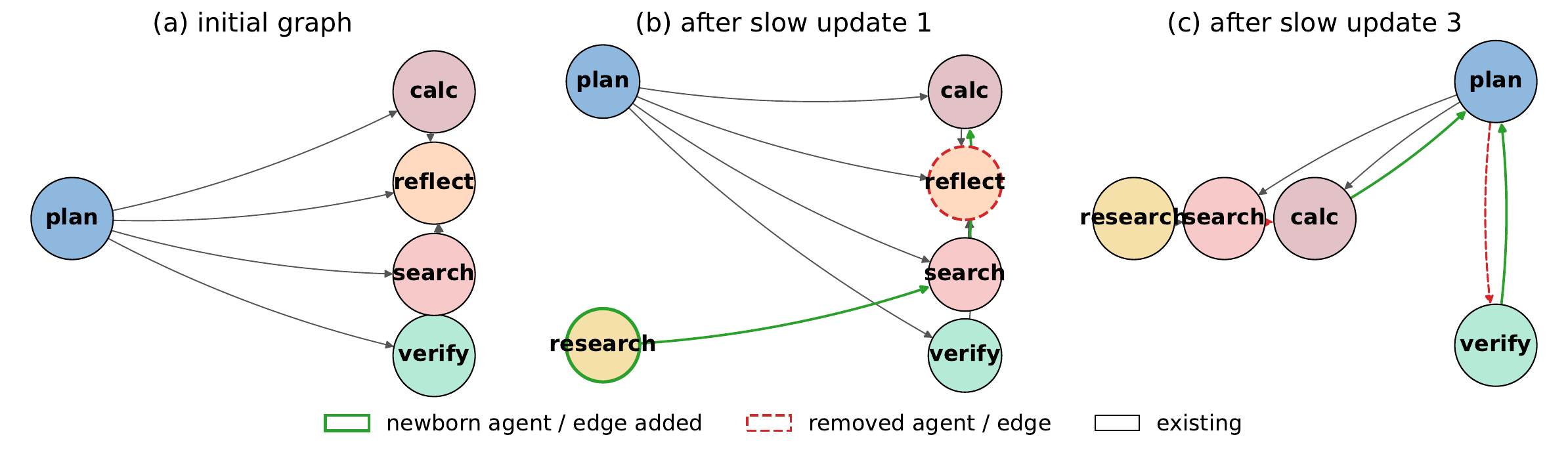}
\caption{\texttt{workbench} trace. }
\label{fig:app-evo-workbench}
\end{figure}

\subsection{Assumption Verification}
\label{app:assump_verify}

Assumption~\ref{ass:llm-bias} states that, conditional on the
round-$t$ history $\mathcal{F}_t$, the meta-LLM rewrite should not
decrease the expected manipulation intensity, up to a small noise term:
$\mathbb{E}[m_{t+1}\mid\mathcal{F}_t]\ge
m_t-\eta\epsilon_{\textrm{noise}}$. We therefore measure the
round-to-round change $\Delta m_t=m_{t+1}-m_t$. The sample mean of
$\Delta m_t$ over all consecutive round pairs provides a direct
empirical estimate of this expected gap. Under the assumption, this
mean should be approximately non-negative. Moreover, if the meta-LLM
approximately follows the discrete replicator update in
Eq.~\eqref{eq:fast}, larger within-round contribution variance should
lead to stronger positive increments.

\begin{figure}
    \centering
    \includegraphics[width=\linewidth]{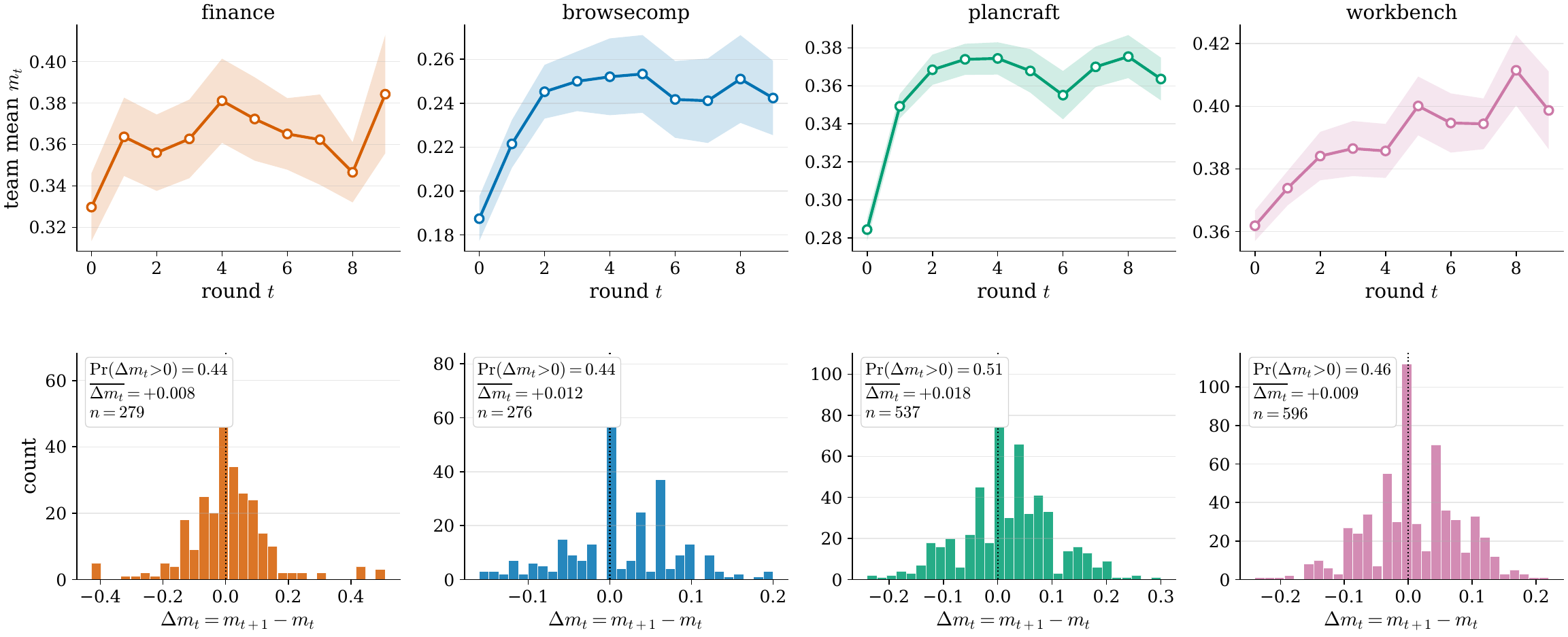}
    \caption{
    Empirical check of the replicator-bias assumption
    (Eq.~\ref{eq:llm-bias}). The top row shows the average team
    contribution $m_t$ across fast rounds, with shaded bands denoting SEM;
    $m_t$ generally increases in the early rounds and then plateaus. The
    bottom row shows the distribution of round-to-round changes
    $\Delta m_t=m_{t+1}-m_t$, whose mean is positive across all benchmarks,
    supporting the predicted non-decreasing trend.
    }
    \label{fig:replicator-regression}
\end{figure}

\paragraph{Result.}
Across all datasets, the average round-to-round increment in
manipulation intensity is positive, supporting the non-decreasing trend
predicted by Assumption~\ref{ass:llm-bias}. The probability of observing
a positive increment is close to one half, with retrieval-heavy datasets
showing slightly more noisy fluctuations and \texttt{plancraft} showing
a somewhat clearer upward tendency. This suggests that negative
increments are mainly due to contribution-score noise rather than a
systematic decline in manipulation intensity. The trajectory plots in
the top row of Figure~\ref{fig:replicator-regression} further show that
$m_t$ typically increases during the early rounds and then gradually
plateaus around the dataset-level mean contribution score, consistent
with the fast phase approaching its within-topology equilibrium before
the slow phase is triggered.



\subsection{Slow-Update Counts}
Figure~\ref{fig:rounds-dist} shows the distribution of slow topology
updates used by \methodname{} on different benchmarks. Easier datasets,
such as \texttt{plancraft} and \texttt{workbench}, usually require no or
only one slow update before termination, indicating that the initial
topology and fast capability refinement are often sufficient. In
contrast, harder datasets, especially \texttt{finance}, more frequently
consume the full slow-update budget. This suggests that complex tasks
require repeated structural reconfiguration, as the system needs to
adapt its communication topology and agent roles across multiple stages
of inference.

\begin{figure}[H]
\centering
\begin{tikzpicture}
\begin{axis}[
  ybar=0pt,
  bar width=4.5pt,
  enlarge x limits=0.06,
  width=0.9\linewidth,
  height=5.0cm,
  ymin=0, ymax=0.65,
  ylabel={Fraction of instances},
  xlabel={Slow updates per instance},
  symbolic x coords={0,1,2,3,4,5,6,7,8,9,10},
  xtick=data,
  legend style={at={(0.5,1.04)}, anchor=south, legend columns=4, draw=none, font=\small},
  tick label style={font=\scriptsize},
  ylabel style={font=\small},
  xlabel style={font=\small},
]
\addplot[fill=blue!25, draw=blue!60!black] coordinates {(0,0.653) (1,0.181) (2,0.050) (3,0.017) (4,0.012) (5,0.009) (6,0.002) (7,0.009) (8,0.003) (9,0.002) (10,0.062)};
\addplot[fill=green!25, draw=green!60!black] coordinates {(0,0.586) (1,0.228) (2,0.071) (3,0.023) (4,0.025) (5,0.013) (6,0.010) (7,0.006) (8,0.003) (9,0.006) (10,0.030)};
\addplot[fill=orange!25, draw=orange!70!black] coordinates {(0,0.430) (1,0.140) (2,0.030) (3,0.080) (4,0.040) (5,0.040) (6,0.020) (7,0.010) (8,0.000) (9,0.000) (10,0.210)};
\addplot[fill=red!25, draw=red!70!black] coordinates {(0,0.240) (1,0.120) (2,0.040) (3,0.020) (4,0.040) (5,0.020) (6,0.000) (7,0.000) (8,0.060) (9,0.040) (10,0.420)};
\legend{plancraft, workbench, browsecomp, finance}
\end{axis}
\end{tikzpicture}
\caption{Distribution of slow-update counts. Easier datasets usually terminate with few slow updates, while harder datasets more often consume the full budget.}
\label{fig:rounds-dist}
\end{figure}
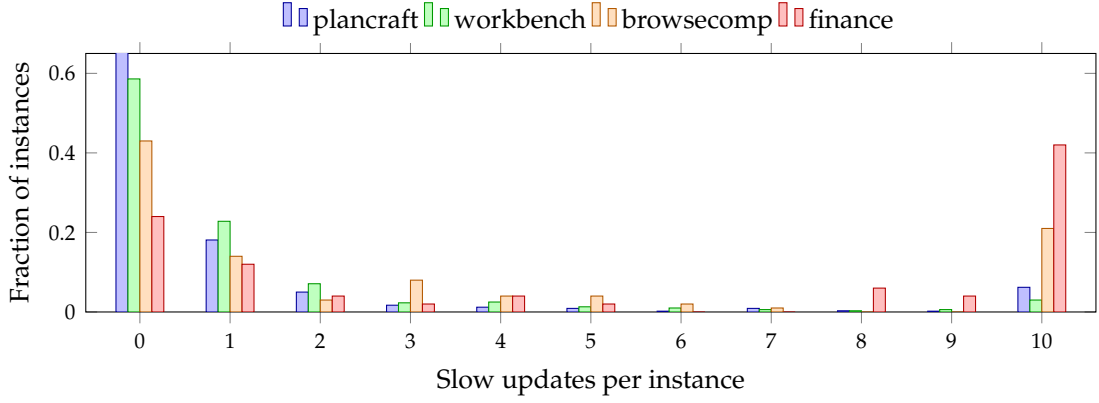

\subsection{Fine-Grained Breakdown}
\label{app:fine-grained}

Table~\ref{tab:fine-grained} reports per-subcategory accuracy and median slow-update count. Subcategories come from dataset metadata: expert time for \texttt{finance}, question type for \texttt{workbench}, answer pattern for \texttt{plancraft}, and expected-answer length for \texttt{browsecomp-plus}. Across datasets, harder subcategories generally require more slow updates.

\begin{table}[h]
\centering
\small
\caption{Per-subcategory accuracy and median slow-update count.}
\label{tab:fine-grained}
\setlength{\tabcolsep}{5pt}
\begin{tabular}{llccc}
\toprule
Dataset & Subcategory & $n$ & mean $\acc$ & median slow upd. \\
\midrule
\texttt{finance} & easy ($\texttt{expert\_mins}\leq5$) & 22 & 0.914$^\dagger$ & 1 \\
\texttt{finance} & hard ($\texttt{expert\_mins}\geq10$) & 28 & 0.651$^\dagger$ & 10 \\
\midrule
\texttt{plancraft} & feasible & 480 & 0.874 & 0 \\
\texttt{plancraft} & impossible & 100 & 0.940 & 0 \\
\midrule
\texttt{workbench} & single: analytics & 120 & 0.838 & 1 \\
\texttt{workbench} & single: calendar & 110 & 0.764 & 0 \\
\texttt{workbench} & single: email & 90 & 0.792 & 0 \\
\texttt{workbench} & single: crm & 80 & 0.681 & 1 \\
\texttt{workbench} & single: project\_mgmt & 80 & 0.825 & 0 \\
\texttt{workbench} & two-tool composite & 150 & 0.907 & 0 \\
\texttt{workbench} & multi-tool ($\geq3$) & 60 & 0.933 & 0 \\
\midrule
\texttt{browsecomp-pl.} & short answer ($\leq10$ chr) & 38 & 0.803 & 0 \\
\texttt{browsecomp-pl.} & medium answer ($11$--$25$) & 45 & 0.733 & 2 \\
\texttt{browsecomp-pl.} & long answer ($>25$) & 17 & 0.647 & 4 \\
\bottomrule
\end{tabular}
\end{table}

\subsection{Stop Reasons}
\label{app:stop-reasons}

Figure~\ref{fig:stop-reasons} reports whether instances stop by reaching the answer-quality threshold or by exhausting the round budget. The budget-exhausted fraction increases with dataset difficulty, suggesting that the controller's own stopping behavior provides an unsupervised difficulty signal.

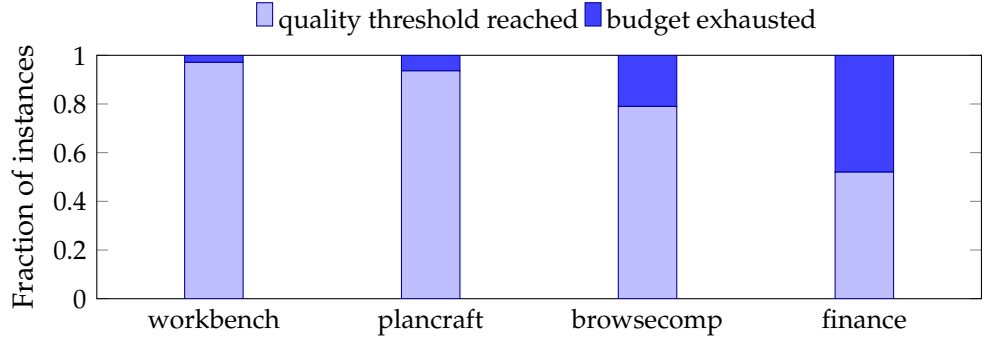
\begin{figure}[h]
\centering
\begin{tikzpicture}
\begin{axis}[
  ybar stacked,
  bar width=22pt,
  width=0.8\linewidth,
  height=4.8cm,
  ymin=0, ymax=1.0,
  ylabel={Fraction of instances},
  symbolic x coords={workbench, plancraft, browsecomp, finance},
  xtick=data,
  enlarge x limits=0.18,
  legend style={at={(0.5,1.05)}, anchor=south, legend columns=2, draw=none, font=\small},
  tick label style={font=\small},
]
\addplot[fill=blue!25, draw=blue!60!black]
  coordinates {(workbench,0.971) (plancraft,0.936) (browsecomp,0.790) (finance,0.520)};
\addplot[fill=blue!75, draw=blue!60!black]
  coordinates {(workbench,0.029) (plancraft,0.064) (browsecomp,0.210) (finance,0.480)};
\legend{quality threshold reached, budget exhausted}
\end{axis}
\end{tikzpicture}
\caption{Stop-reason distribution. Harder datasets have a larger budget-exhausted tail.}
\label{fig:stop-reasons}
\end{figure}

\subsection{Graph Densification}
\label{app:densification}

\methodname{} edits both nodes and edges. Figure~\ref{fig:graph-growth} shows that node count stays relatively stable, while edge count grows consistently. Thus, the slow loop mainly rewires existing agents rather than spawning many new ones, supporting the bounded-edit design.

\begin{figure}[h]
\centering
\ref{graph-growth-legend}\\[2pt]
\begin{subfigure}[t]{0.49\linewidth}
\centering
\begin{tikzpicture}
\begin{axis}[
  width=\linewidth,
  height=5cm,
  xlabel={Slow-update step},
  ylabel={Avg.\ node count},
  legend to name=graph-growth-legend,
  legend columns=4,
  legend style={font=\small, draw=none, column sep=8pt},
  tick label style={font=\small},
  ylabel style={font=\small},
  xlabel style={font=\small},
  xmin=1, xmax=10,
  ymin=4.8, ymax=6.0,
  grid=major, grid style={gray!25, dashed},
]
\addplot coordinates {(1,5.00) (2,5.03) (3,5.07) (4,5.10) (5,5.15) (6,5.12) (7,5.12) (8,5.12) (9,5.14) (10,5.14)};
\addplot coordinates {(1,5.00) (2,4.91) (3,4.88) (4,5.00) (5,4.96) (6,5.04) (7,5.00) (8,4.95) (9,5.02) (10,5.05)};
\addplot coordinates {(1,5.00) (2,5.02) (3,5.00) (4,4.98) (5,4.98) (6,5.02) (7,5.02) (8,5.05) (9,5.04) (10,5.06)};
\addplot coordinates {(1,5.00) (2,5.19) (3,5.34) (4,5.34) (5,5.40) (6,5.53) (7,5.58) (8,5.70) (9,5.79) (10,5.86)};
\legend{finance, browsecomp, plancraft, workbench}
\end{axis}
\end{tikzpicture}
\caption{Node count.}
\label{fig:node-growth}
\end{subfigure}
\hfill
\begin{subfigure}[t]{0.49\linewidth}
\centering
\begin{tikzpicture}
\begin{axis}[
  width=\linewidth,
  height=5cm,
  xlabel={Slow-update step},
  ylabel={Avg.\ edge count},
  tick label style={font=\small},
  ylabel style={font=\small},
  xlabel style={font=\small},
  xmin=1, xmax=10,
  ymin=4.8, ymax=9.5,
  grid=major, grid style={gray!25, dashed},
]
\addplot coordinates {(1,5.00) (2,5.40) (3,5.80) (4,6.10) (5,6.40) (6,6.60) (7,6.80) (8,6.95) (9,7.02) (10,7.10)};
\addplot coordinates {(1,6.00) (2,6.10) (3,6.20) (4,6.30) (5,6.35) (6,6.40) (7,6.45) (8,6.48) (9,6.50) (10,6.52)};
\addplot coordinates {(1,6.00) (2,6.40) (3,6.80) (4,7.15) (5,7.40) (6,7.60) (7,7.75) (8,7.88) (9,7.95) (10,8.03)};
\addplot coordinates {(1,6.00) (2,6.60) (3,7.20) (4,7.70) (5,8.10) (6,8.40) (7,8.65) (8,8.90) (9,9.10) (10,9.24)};
\end{axis}
\end{tikzpicture}
\caption{Edge count.}
\label{fig:edge-growth}
\end{subfigure}
\caption{Average graph size over slow-update steps. Edge counts grow more consistently than node counts, indicating that the slow loop primarily performs rewiring.}
\label{fig:graph-growth}
\end{figure}
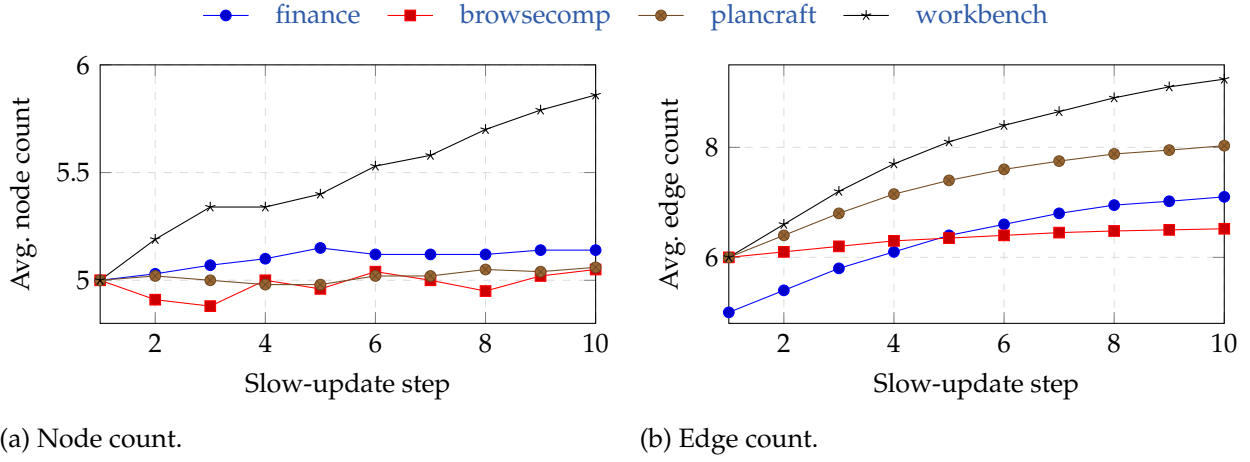

\subsection{Case Study: \texttt{finance} Instance 17}
\label{app:case-study-finance17}

We trace one hard \texttt{finance} case asking whether Workday reports a gross or net retention metric. The initial graph fails because retrieval repeatedly cites secondary sources instead of primary filings. Across slow updates, the meta-controller removes irrelevant or failing retrieval-side agents and introduces increasingly specialized search roles, eventually finding the 10-K disclosure and enabling the verifier to accept the answer.

\begin{table}[h]
\centering
\scriptsize
\caption{Meta-controller decisions on \texttt{FIN0017}.}
\label{tab:case-study-fin17}
\setlength{\tabcolsep}{4pt}
\begin{tabular}{p{0.05\linewidth}p{0.36\linewidth}p{0.26\linewidth}p{0.25\linewidth}}
\toprule
Step & Rationale & Birth/death & Edge edits \\
\midrule
1 & Retrieval bottleneck; calculator irrelevant. & $-$ calculator; $+$ researcher & re-type searcher$\to$verifier \\
2 & Searcher stagnates; evidence not reaching verifier. & $-$ searcher; $+$ researcher\#2 & add researcher$\to$verifier \\
3 & Researchers cite secondary sources. & $-$ researcher\#2; $+$ primary-filing searcher & add planner$\to$reflector \\
4 & Need stricter 10-K/10-Q retrieval. & $-$ searcher'; $+$ 10-K/10-Q searcher & none \\
\bottomrule
\end{tabular}
\end{table}

This case illustrates topology-capability coupling: the controller repeatedly diagnoses the same retrieval bottleneck but escalates specialization rather than applying unrelated edits.

\subsection{Contrast Case: \texttt{finance} Instance 8}
\label{app:case-study-finance8}

We also include an easier finance case asking how many basis points MU beat or missed its Q3 2024 GAAP gross-margin guidance. The initial graph fails to retrieve the guidance number, but a single slow update removes the irrelevant calculator and adds a researcher specialized in extracting named figures from filings. The next round retrieves the relevant actual and guidance margins, computes the difference, and stops. This contrast shows that \methodname{} spends evolution budget only where the instance requires it: some tasks need repeated specialization, while others are solved after one targeted structural edit.
\section{All Prompts}
\label{app:prompts}

This appendix lists every prompt used by \methodname\ verbatim. Curly
braces denote Python-style format slots filled in at runtime.

\subsection{Meta-LLM Prompt}

The meta-LLM is invoked once per slow update. Its input bundles
the task profile, current graph, per-agent fast-round summaries, and
global round scores; its output is a structured JSON object.

\begin{lstlisting}[language=prompt, caption={Meta-LLM system prompt.}, label={lst:meta-system}]
# System
You are the meta-controller of a multi-agent LLM system that solves
a single task instance by evolving its topology and the capabilities
of its agents at inference time.

On every invocation ("slow update"), you receive:
  1. the task description and a running task profile,
  2. the current agent graph as JSON,
  3. fast-round traces per agent (tool calls, critiques, round reward),
  4. the most recent global round scores.

You output a structured JSON object with the following fields:
  - birth_death_pairs: at most 2 pairs (v_dead, v_new_spec)
  - graph_edit: at most 4 (edge_add | edge_remove) operations
  - graph_diff: the implied graph delta, for provenance
  - agent_feedback: per-agent capability deltas (prompt edits, memory seeds)
  - global_rationale: short free-text reasoning, <=3 sentences
  - time_control: one of {continue, slow_again, stop}

Hard constraints:
  - do NOT delete the sink agent;
  - do NOT spawn >2 agents per slow update;
  - do NOT edit >4 edges per slow update;
  - prefer edits that re-route evidence over edits that spawn new roles;
  - if round scores are monotone-improving, choose time_control=continue.
\end{lstlisting}

\begin{lstlisting}[language=prompt, caption={Meta-LLM developer prompt.}, label={lst:meta-dev}]
# Developer
Follow this procedure:
  1. Read the task profile and identify the missing evidence slots.
  2. Inspect per-agent reward trajectories. Mark agents with
     reward < 0.2 for possible death; mark evidence slots with no
     assigned agent for possible birth.
  3. If a single agent is handling >2 distinct sub-goals, propose
     splitting via a birth/death pair.
  4. Re-wire edges so each evidence slot has a clear producer ->
     consumer path to the sink.
  5. Emit the JSON output.
\end{lstlisting}

\begin{lstlisting}[language=prompt, caption={Meta-LLM user payload (runtime format).}, label={lst:meta-user}]
# User
Task: {task_description}

Current graph:
{graph_json}

Per-agent fast-round summaries:
{per_agent_summaries}

Recent round scores (oldest to newest):
{round_scores}

Provide the JSON object now.
\end{lstlisting}

\subsection{Fast-round Agent Prompt}

Each agent runs with a role-specific profile embedded in the system
prompt. The fast-round reflection operator $\mathcal{F}_{\text{fast}}$
appends a self-critique block after each round.

\begin{lstlisting}[language=prompt, caption={Agent system prompt template.}, label={lst:agent-system}]
# System
You are the {role_name} agent.

Goal: {role_goal}
Constraint: {role_constraint}

You have access to the following tools:
{tool_specs}

Incoming messages (from upstream agents):
{incoming_messages}

Produce an output message that advances the task. If you are a sink
agent, end with a line beginning "Final Answer:".
\end{lstlisting}

\begin{lstlisting}[language=prompt, caption={Fast-round self-reflection prompt (applied after each round).}, label={lst:reflect}]
# User
You just completed round {round_idx}. Your reward this round was
{reward}. Here is your round output:

{round_output}

In 2-3 sentences, identify ONE concrete capability update (prompt
edit, memory seed, or tool-use change) that would improve your next
round. Return only the update, no apology.
\end{lstlisting}

\subsection{Meta-judge Prompt}

\begin{lstlisting}[language=prompt, caption={Rubric judge prompt.}, label={lst:judge}]
# User
You are an expert grader. Evaluate a candidate answer using the
rubric below.

Candidate Answer:
"{answer}"

Rubric (JSON checklist):
{rubric_json}

For each rubric item:
  - If operator is "correctness": does the answer satisfy/support the
    criterion? (true/false)
  - If operator is "contradiction": does the answer contradict the
    criterion? (true/false)

Return JSON only:
{"results":[true,false,...]}
\end{lstlisting}

\begin{lstlisting}[language=prompt, caption={Contribution-score judge prompt}, label={lst:meta-schema}]
{
  You are an evaluator for multi-agent collaboration. Score how much this agent output 
  can contribute to solving the task.
  Return strict JSON only: {"score": <0~1 float>, "reason": "..."}.
  Scoring rubric:
  - 0.0: irrelevant/wrong/no useful signal
  - 0.3: weak but partially relevant
  - 0.5: moderately useful evidence or decomposition
  - 0.7: strong useful contribution with concrete progress
  - 1.0: critical and directly decisive contribution

  {evidence_note}

  Task:
  {self._instance_text()}

  Round Query Given To Agent:
  {query}

  Agent ID: {agent_id}
  Agent Role: {role}

  Tools Used: {tool_names}                                                             
  Evidence Gate Precheck: {PASS_or_FAIL}                                               
  Agent Output:                                                                        
  {output_for_judge}                                                                   
}
\end{lstlisting}

\subsection{Dataset-specific Task Prompt Fragments}

The full dataset-specific templates are in
\texttt{prompts/dataset-shared/} in the released code. We reproduce
the task-instance fragment for each here.

\begin{lstlisting}[language=prompt, caption={Finance task template.}, label={lst:finance-task}]
You are a financial agent. You are given a question and you need to
answer it using the tools provided. Assume the current date is
April 07, 2025.

You will have access to a data storage system. You can use this
system to store parsed contents of HTML pages retrieved from the
web. You can then use the retrieve_information tool to answer
questions or gather information from the stored documents.

When you have the final answer, call the `submit_final_result` tool.

Question:
{question}
\end{lstlisting}

\begin{lstlisting}[language=prompt, caption={Browsecomp task template.}, label={lst:browsecomp-task}]
You are a research assistant. Use search_documents and
retrieve_document to find the answer to the question below. Answer
with the full birth name only, nothing else. Call `done(answer,
confidence_score)` to submit.

Question:
{question}
\end{lstlisting}

\begin{lstlisting}[language=prompt, caption={Plancraft task template.}, label={lst:plancraft-task}]
Determine whether the target crafting goal is achievable given the
inventory. Produce a valid minimal action sequence, or output exactly
IMPOSSIBLE. Tools: search, move, smelt, impossible.

Target: {target_item}
Inventory:
{inventory}
\end{lstlisting}

\begin{lstlisting}[language=prompt, caption={Workbench task template.}, label={lst:workbench-task}]
Execute the workplace task below. When complete, call `done`.

Task:
{question}
\end{lstlisting}

\subsection{Output JSON Schema for the Meta-LLM}

\begin{lstlisting}[language=prompt, caption={Meta-LLM output JSON schema.}, label={lst:meta-schema}]
{
  "birth_death_pairs": [
    { "v_dead": "<agent_id or null>",
      "v_new": { "role": "<role name>", "goal": "<...>", "tools": [...] } }
  ],
  "graph_edit": [
    { "op": "edge_add" | "edge_remove", "from": "<id>", "to": "<id>" }
  ],
  "graph_diff": { "nodes_added": [...], "nodes_removed": [...],
                   "edges_added": [...], "edges_removed": [...] },
  "agent_feedback": {
    "<agent_id>": { "prompt_delta": "<text>", "memory_seed": "<text>" }
  },
  "global_rationale": "<<=3 sentences>",
  "time_control": "continue" | "slow_again" | "stop"
}
\end{lstlisting}

\end{document}